%% file: main.tex
\documentclass{article}

\PassOptionsToPackage{numbers, compress}{natbib}

\usepackage[final]{neurips_2023}

\usepackage[utf8]{inputenc} 
\usepackage[T1]{fontenc}    
\usepackage{hyperref}       
\usepackage{url}            
\usepackage{booktabs}       
\usepackage{amsfonts}       
\usepackage{nicefrac}       
\usepackage{microtype}      
\usepackage{xcolor}         

\usepackage[toc,page]{appendix}
\usepackage{etoc}

\usepackage[inline]{enumitem}
\newlist{inlinelist}{enumerate*}{1}
\setlist[inlinelist]{label=(\roman*)}

\usepackage{graphicx}

\usepackage{amsmath}
\usepackage{amssymb}
\usepackage{mathtools}

\usepackage{amsthm}
\usepackage{multirow}

\usepackage{setspace}
\usepackage[capitalize,noabbrev]{cleveref}

\theoremstyle{plain}
\newtheorem{theorem}{Theorem}[section]
\newtheorem{proposition}[theorem]{Proposition}
\newtheorem{lemma}[theorem]{Lemma}

\theoremstyle{definition}
\newtheorem{definition}[theorem]{Definition}

\theoremstyle{remark}
\newtheorem{remark}[theorem]{Remark}

\usepackage{wrapfig,lipsum,booktabs}

\usepackage{algorithm}
\usepackage{algpseudocode}

\usepackage{threeparttable}

\newcommand{\argmax}[1] {\underset{#1}{\arg \max }} 
\newcommand{\argmin}[1] {\underset{#1}{\arg \min }} 
\newcommand{\s} {\mathbf{s}}
\newcommand{\act} {\mathbf{a}}
\newcommand*\diff{\mathop{}\!\mathrm{d}}

\definecolor{myblue}{rgb}{0.87,0.92,0.96}
\definecolor{mygray}{rgb}{0.859,0.859,0.859}

\crefname{theorem}{Theorem}{Theorems}
\crefname{proposition}{Proposition}{Propositions}
\crefname{lemma}{Lemma}{Lemmas}
\crefname{corollary}{Corollary}{Corollaries}
\crefname{definition}{Definition}{Definitions}
\crefname{assumption}{Assumption}{Assumptions}
\crefname{remark}{Remark}{Remarks}

\title{Train Once, Get a Family: State-Adaptive Balances for Offline-to-Online Reinforcement Learning\thanks{$^*$ Equal contribution. Email: \texttt{\{wsz21, yangqs19\}@mails.tsinghua.edu.cn}.}\thanks{$^\dag$ Corresponding author. Email: \texttt{gaohuang@tsinghua.edu.cn}.}}

\renewcommand\footnotemark{}

\author{%
  Shenzhi Wang$^{1*}$, Qisen Yang$^{1*}$, Jiawei Gao$^{1}$, Matthieu Lin$^{2}$, Hao Chen$^{4}$, Liwei Wu$^{1}$ \\ \textbf{Ning Jia$^{3}$, Shiji Song$^{1}$, Gao Huang$^{1\dag}$}
  \\
  $^{1}$ Department of Automation, BNRist, Tsinghua University \\
  $^{2}$ Department of Computer Science, BNRist, Tsinghua University \\
  $^{3}$ Beijing Academy of Artificial Intelligence (BAAI) \qquad 
  $^{4}$ Independent Researcher
}
\vspace{-20pt}

\begin{document}

\maketitle

\etocdepthtag.toc{mtchapter}
\etocsettagdepth{mtchapter}{subsection}
\etocsettagdepth{mtappendix}{none}

\begin{center}
    \vspace{-30pt}
    Project Page: \url{https://shenzhi-wang.github.io/NIPS_FamO2O}
    \vspace{10pt}
\end{center}

\begin{abstract}
 \input{sections/0_abstract}
\end{abstract}

\begin{figure}[!h]
    \centering
    \includegraphics[width=.97\textwidth]{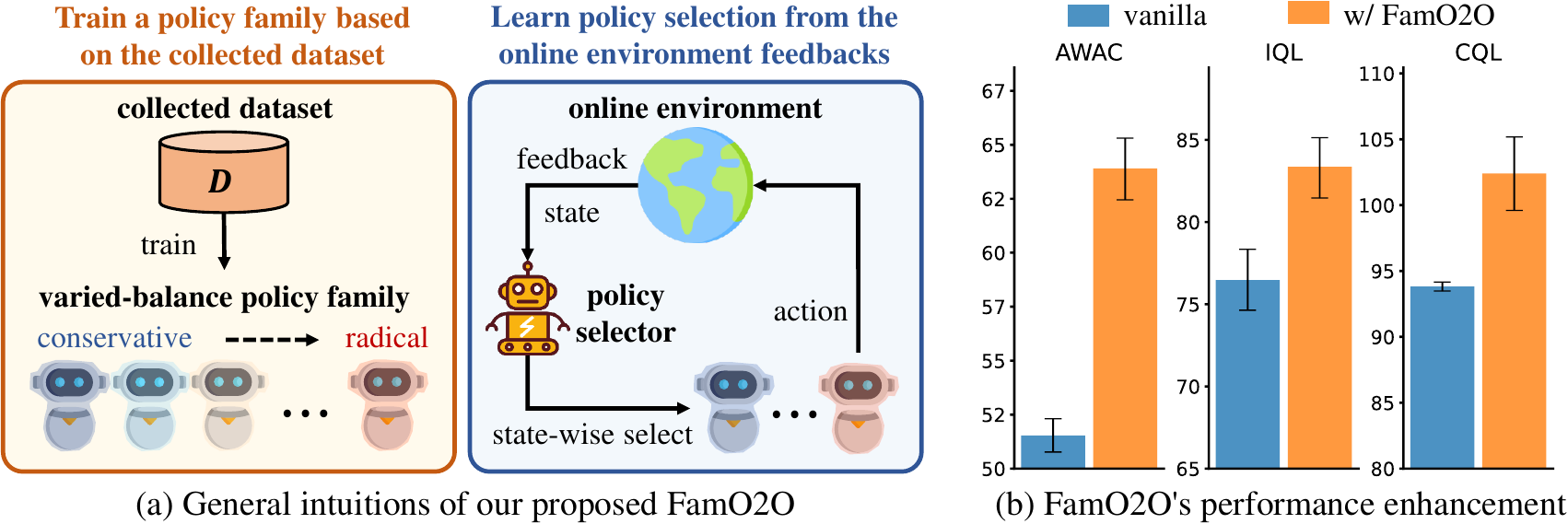}
    \caption{FamO2O trains a policy family from datasets and selects policies state-adaptively using online feedback. Easily integrated, FamO2O statistically enhances existing algorithms' performance.}
    \label{fig:teaser}
    \vspace{10pt}
\end{figure}

\section{Introduction}
\input{sections/1_introduction}

\vspace{-5pt}
\section{Preliminaries}
\input{sections/2_preliminaries}

\input{sections/3_method}

\section{Experimental Evaluation}
\label{sec: experimental evaluation}
\input{sections/4_experiments}

\section{Discussion}
\label{sec: discussion}
\input{sections/5_discussion}

\section{Conclusion}\label{sec: conclusion}
\input{sections/7_conclusion}

\section*{Acknowledgement}
This work was supported in part by the Key-Area Research and Development Program of Guangdong Province under Grant 2020B1111500002, the National Natural Science Foundation of China under Grants 62022048 and 62276150, and the Guoqiang Institute of Tsinghua University.

\newpage
{
\small
\bibliographystyle{plain}
\bibliography{ref.bib}
}

\newpage
\appendix
\begin{appendices}
\etocdepthtag.toc{mtappendix}
\etocsettagdepth{mtchapter}{none}
\etocsettagdepth{mtappendix}{subsection}
\begin{spacing}{0.5}
\tableofcontents
\end{spacing}

\input{sections/8_appendix}
\end{appendices}

\end{document}

%% file: sections/0_abstract.tex
Offline-to-online reinforcement learning (RL) is a training paradigm that combines pre-training on a pre-collected dataset with fine-tuning in an online environment.
However, the incorporation of online fine-tuning can intensify the well-known distributional shift problem.
Existing solutions tackle this problem by imposing a \emph{policy constraint} on the \emph{policy improvement} objective in both offline and online learning.
They typically advocate a single balance between policy improvement and constraints across diverse data collections. 
This one-size-fits-all manner may not optimally leverage each collected sample due to the significant variation in data quality across different states.
To this end, we introduce Family Offline-to-Online RL (FamO2O), a simple yet effective framework that empowers existing algorithms to determine state-adaptive improvement-constraint balances.
FamO2O utilizes a \emph{universal model} to train a family of policies with different improvement/constraint intensities, and a \emph{balance model} to select a suitable policy for each state.
Theoretically, we prove that state-adaptive balances are necessary for achieving a higher policy performance upper bound.
Empirically, extensive experiments show that FamO2O offers a statistically significant improvement over various existing methods, achieving state-of-the-art performance on the D4RL benchmark.
Codes are available at \url{https://github.com/LeapLabTHU/FamO2O}.

%% file: sections/1_introduction.tex
\vspace{-5pt}
Offline reinforcement learning (RL) provides a pragmatic methodology for acquiring policies utilizing pre-existing datasets, circumventing the need for direct environment interaction~\cite{Sergey2020tutorial}.
Nonetheless, the attainable policy performance in offline RL is frequently constrained by the quality of the dataset~\cite{td3+bc+finetune}.
The offline-to-online RL paradigm addresses this limitation by refining the offline RL policy through fine-tuning in an online setting~\cite{nair2020awac}.

While online fine-tuning can indeed elevate policy performance, it also potentially exacerbates the issue of distributional shift~\cite{Sergey2020tutorial}, where policy behavior diverges from the dataset distribution. Such shifts typically ensue from drastic policy improvements and are further amplified by state~distribution changes when transitioning from offline learning to online fine-tuning~\cite{fujimoto2019BCQ, kumar2019stabilizing, fu2019diagnosing, kumar2020discor}. Prior works have attempted to counter this by imposing policy constraints on the policy improvement objective to deter excessive exploration of uncharted policy space~\cite{kumar2019stabilizing, wu2019behavior, nair2020awac}. However, this conservative approach can inadvertently stifle policy improvement~\cite{nair2020awac}. In essence, offline-to-online RL necessitates an effective balance between \emph{policy improvement} and \emph{policy constraint} during policy optimization.

\begin{wrapfigure}{r}{0.62\textwidth}
\vspace{-15pt}
\begin{center}
\includegraphics[width=0.62\textwidth]{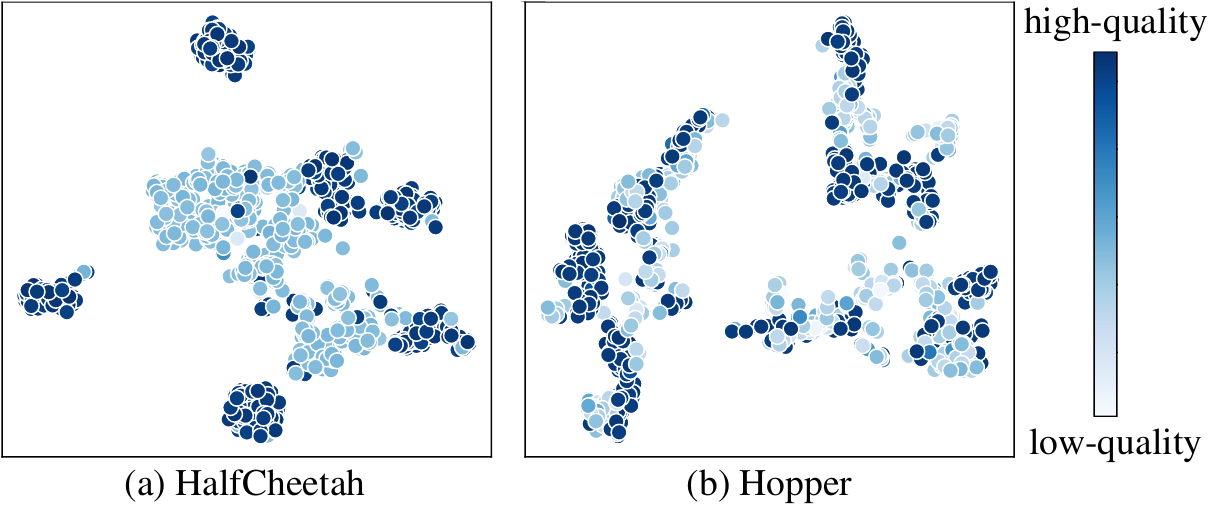}
\end{center}
    \vspace{-5pt}
    \caption{A t-SNE~\cite{tSNE} \textbf{visualization of randomly selected states} from (a) HalfCheetah and (b) Hopper medium-expert datasets in D4RL~\cite{fu2020d4rl}. The color coding represents the return of the trajectory associated with each state. This visualization underscores \textbf{the significant variation in data quality across different states}.}
    \vspace{-10pt}
    \label{fig:different-state-quality}
\end{wrapfigure}


Regrettably, prior offline-to-online RL algorithms tend to adopt a monolithic approach towards this improvement-constraint trade-off, indiscriminately applying it to all data in a mini-batch~\cite{td3+bc, td3+bc+finetune, wu2022supported} or the entire dataset~\cite{nair2020awac, iql, kumar2020cql, lyu2022mildly}. Given the inherent data quality variation across states (see \cref{fig:different-state-quality}), we argue that this one-size-fits-all manner may fail to optimally exploit each sample. In fact, data yielding high trajectory returns should encourage more ``conservative" policies, while data leading to poor returns should incite more ``radical" policy improvement.

In this paper, we introduce a novel framework, Family Offline-to-Online RL (FamO2O), which can discern a state-adaptive improvement-constraint balance for each state. FamO2O's design is founded on two key insights delineated in \cref{fig:teaser}(a): 
\begin{inlinelist}
\item The collected dataset, abundant in environmental information, could facilitate the training of a diverse policy family ranging from conservative to radical, and 
\item feedback from online learning might assist in selecting an appropriate policy from this family for each state. 
\end{inlinelist}
As depicted in \cref{fig: framework}, FamO2O incorporates a \emph{universal model} and a \emph{balance model}. The universal model, conditioned on a \emph{balance coefficient}, determines the degree of policy conservatism or radicalism, while the balance model learns a mapping from states to balance coefficients, aiding the universal model in tailoring its behavior to each specific state.

FamO2O represents, to the best of our knowledge, the first offline-to-online RL approach harnessing the power of state-adaptive improvement-constraint balances. 
Theoretically, we establish that in policy optimization, \emph{point-wise KL constraints} afford a superior performance upper bound compared to the \emph{distributional KL constraints} adopted in prior works~\cite{awr, nair2020awac, iql}. Importantly, these state-adaptive balances become indispensable when addressing point-wise KL constraints, thereby underlining the necessity of incorporating such balances in offline-to-online RL.
Experimental results, as summarized in \cref{fig:teaser}(b), reveal that FamO2O is a simple yet effective framework that statistically significantly improves various offline-to-online RL algorithms and achieves state-of-the-art performance.

%% file: sections/2_preliminaries.tex
\vspace{-12pt}
We introduce essential RL and offline-to-online RL concepts and notations here. 
For descriptive convenience and theoretical analysis, we use the advantaged-weight regression (AWR) algorithm framework~\cite{awr, nair2020awac, iql}, but \textbf{FamO2O isn't limited to the AWR framework}.
We later demonstrate its integration with non-AWR algorithms in \cref{sec: exp on other algorithms} and \cref{sec: FamO2O extension}.

\vspace{-5pt}
\subsection{Reinforcement learning formulation} 
\vspace{-5pt}
RL is typically expressed as a Markov decision process (MDP)~\cite{sutton1998}, denoted as $(\mathcal{S}, \mathcal{A}, P, d_0, R, \gamma)$. Here, $\mathcal{S}$ and $\mathcal{A}$ are the state\footnote{For simplicity, we use ``state'' and ``observation'' interchangeably in fully or partially observed environments.} and action space; $P(\s_{t+1}|\s_t, \act)$ is the environmental state transition probability; $d_0(\s_0)$ represents initial state distribution; $R(\s_t, \act_t, \s_{t+1})$ is the reward function; and $\gamma\in (0, 1]$ is the discount factor.

\subsection{Offline-to-online reinforcement learning}
Offline-to-online RL is a training paradigm including two phases:
\begin{inlinelist}
    \item \textit{offline pre-training}: pre-training a policy based on an offline dataset;
    \item \textit{online fine-tuning}: fine-tuning the pre-trained policy by interacting with the environment.
\end{inlinelist}
Note that in the offline pre-training phase, the policy cannot interact with the environment, but in the online fine-tuning phase, the policy has access to the offline dataset.

Similar to offline RL algorithms, offline-to-online RL algorithms' training objectives usually consist of two terms, either explicitly~\cite{nair2020awac, iql} or implicitly~\cite{kumar2020cql, lee2022pessimistic}:
\begin{inlinelist}
\item \textit{policy improvement}, which aims to optimize the policy according to current value functions;
\item \textit{policy constraint}, which keeps the policy around the distribution of the offline dataset or current replay buffer.
\end{inlinelist}
Using the AWR algorithm framework~\cite{awr, nair2020awac, iql}, we demonstrate our method (also applicable to non-AWR algorithms) and define notations in \cref{eq: training objective} for later use.
%
\begin{equation}
        L_\pi=\mathbb{E}_{(\s, \act) \sim \mathcal{D}} \Bigg[
        \overbrace{\exp \Big(\underbrace{\beta}_{\text{balance coefficient}}\big(\underbrace{Q(\s, \act)-V(\s)}_{\text{policy improvement}}\big)\Big)}^{\text{imitation weight}} \cdot \underbrace{\log \pi(\act|\s)}_{\text{policy constraint}}\Bigg].
        \label{eq: training objective}
\end{equation}

$L_\pi$ is a maximization objective, while $\mathcal{D}$ represents the collected dataset. Initially, during offline pre-training, $\mathcal{D}$ starts with the pre-collected offline dataset $\mathcal{D}_{\text{offline}} = \{(\s_k, \act_k, \s^{\prime}_k, r_k) \mid k=1, 2, \cdots, N\}$. As we move to the online fine-tuning phase, online interaction samples are continuously incorporated into $\mathcal{D}$~\cite{nair2020awac,iql}. The balance coefficient $\beta$ is a predefined hyperparameter moderating between the policy improvement and policy constraint terms, while the imitation weight sets the imitation intensity for the state-action pair $(s, a)$.

%% file: sections/3_method.tex
\section{State-Adaptive Balance Coefficients}
\label{sec: rationality of state-adaptive balance coefficients}
\vspace{-5pt}
Our design of state-adaptive improvement-constraint balances is motivated by the observations that
\begin{inlinelist}
    \item the quality of the dataset's behavior, \textit{i.e.}, the trajectory returns,  fluctuates greatly with different states, as shown in \cref{fig:different-state-quality}; 
    \item state-dependent balances are conducive to a higher performance upper bound.
\end{inlinelist}
In this section, we will theoretically validate the latter point.

We first present the policy optimization problem with point-wise KL constraints in \cref{def: point-wise constraint optimization problem definition}, which is the focus of FamO2O:

\begin{definition}[Point-wise KL Constrained Optimization Problem]
\label{def: point-wise constraint optimization problem definition}
We consider a policy optimization problem defined as follows:
%
\begin{align}
\max_{\pi} & \ \ \mathbb{E}_{\s\sim d{\pi_\beta}(\cdot), \act \sim \pi(\cdot|\s)} \left[Q^{\pi^k}(\s, \act)-V^{\pi^k}(\s)\right] \label{eq: optimization problem}\\
& \text { s.t. } D_{\mathrm{KL}}(\pi(\cdot | \s) | \pi_{\beta}(\cdot | \s)) \leq \epsilon_\s, \quad \forall \s \in \mathcal{S} \label{eq: kl constraint} \\
& \int_{\act\in\mathcal{A}} \pi(\act | \s)\diff\act =1, \quad \forall \s \in \mathcal{S} \label{eq: normalization constraint}.
\end{align}

Here, $\pi^k (k \in \mathbb{N})$ denotes the policy at iteration $k$, $\pi_\beta$ signifies a behavior policy representing the action selection way in the collected dataset $\mathcal{D}$, $d_{\pi_{\beta}}(\s)$ refers to the state distribution of $\pi_\beta$, and $\epsilon_\s$ is a state-related constant. The optimal policy derived from \crefrange{eq: optimization problem}{eq: normalization constraint} is designated as $\pi^{k+1}$.
\end{definition}

The optimization problem common in previous work~\cite{awr, nair2020awac, iql} is shown in \cref{def: distributional constraint optimization problem definition}:
\begin{definition}[Optimization problem with distributional KL constraints]
\label{def: distributional constraint optimization problem definition}
The definition of the policy optimization problem with distributional KL constraints is the same as \cref{def: point-wise constraint optimization problem definition}, except that \cref{eq: kl constraint} in \cref{def: point-wise constraint optimization problem definition} is substituted by \cref{eq: distributional constraint}, where $\epsilon$ is a constant:
%
\begin{equation}
      \int_{\s\in\mathcal{S}} d_{\pi_\beta}(\s) D_{\mathrm{KL}}(\pi(\cdot | \s) \| \pi_{\beta}(\cdot | \s)) \diff\s \le \epsilon \label{eq: distributional constraint}.
      \vspace{5pt}
\end{equation} 
\end{definition}

\begin{remark}
    The update rule in \cref{eq: training objective} is based on the optimization problem in \cref{def: distributional constraint optimization problem definition}.
\end{remark}

The point-wise constraints' superiority over distributional constraints is shown in \cref{prop: point-wise superiority}:
\begin{proposition}[Advantage of point-wise KL constraints] \label{prop: point-wise superiority}
 Denote the optimal value in \cref{def: point-wise constraint optimization problem definition} as $J^k_*[\{\epsilon_\s, \s\in\mathcal{S}\}]$, the optimal value in \cref{def: distributional constraint optimization problem definition} as $J^k_*[\epsilon]$.
These optimal values satisfy:
\begin{equation}
    \forall \epsilon \ge 0, \quad \exists \{\epsilon_\s, \s \in \mathcal{S}\}, \quad J^k_*[\{\epsilon_\s, \s\in\mathcal{S}\}] \ge J^k_*[\epsilon].
\end{equation}
\end{proposition}
\begin{proof}
Please refer to \cref{appendix: proof of point-wise superiority}.
\end{proof}

\cref{prop: point-wise superiority} indicates that the optimal value under the point-wise KL constraints, given suitable point-wise constraints, is \emph{no less than} that under distributional KL constraints. This finding justifies our approach under point-wise constraints.

\cref{prop: State-dependent balance coefficient} shows the necessity of state-dependent balance coefficient design in solving the point-wise KL constraint optimization problem:
\begin{proposition}[State-dependent balance coefficient]
\label{prop: State-dependent balance coefficient}
Consider the optimization problem in \cref{def: point-wise constraint optimization problem definition}.
Assume that the state space $\mathcal{S} = [s_{\min}, s_{\max}]^l$ ($l$ is the state dimension), and the feasible space constrained by \crefrange{eq: kl constraint}{eq: normalization constraint} is not empty for every $\s \in \mathcal{S}$.
Then the optimal solution of $\pi^{k+1}$, denoted as $\pi_*^{k+1}$, satisfies that $\forall \s \in \mathcal{S}, \act \in \mathcal{A}$,
\begin{equation}
    \pi_*^{k+1}(\act|\s) = \frac{\pi_\beta(\act|\s)}{Z_{\s}} \exp\left({\color{blue}\beta_\s}(Q^{\pi^k}(\s, \act)-V^{\pi^k}(\s))\right),
    \label{eq: adaptive optimal policy}
\end{equation}
%
where $\beta_{\s}$ is a state-dependent balance coefficient, and $Z_{\s}$ is a normalization term.
When utilizing a parameterized policy $\pi_{\mathbf{\phi}}$ to approximate the optimal policy $\pi_*^{k+1}$, the training objective can be formulated as:
%
\begin{equation}
\label{eq: adaptive training objective}
    \mathbf{\phi} = \argmax{\phi}\mathbb{E}_{(\s, \act)\sim \mathcal{D}} \Big[ \exp({\color{blue}\beta_{\s}} (Q^{\pi^k}(\s, \act)-V^{\pi^k}(\s))) \log \pi_{\mathbf{\phi}}(\act|\s) \Big].
\end{equation}

\end{proposition}
\begin{proof}
    The proof is deferred to \cref{appendix: proof of state-dependent balance coefficient}.
\end{proof}

In contrast to AWR~\cite{awr} and AWAC~\cite{nair2020awac}, \cref{prop: State-dependent balance coefficient} highlights state-dependent (marked in {\color{blue}blue}) balance coefficients in \crefrange{eq: adaptive optimal policy}{eq: adaptive training objective}, as opposed to a pre-defined hyperparameter in \cref{eq: training objective}. This state-adaptiveness is due to \cref{prop: State-dependent balance coefficient} considering the finer-grained constraints in \cref{def: point-wise constraint optimization problem definition}. Together, \cref{prop: point-wise superiority} and \cref{prop: State-dependent balance coefficient} indicate state-adaptive balance coefficients contribute to a higher performance upper bound.

\section{Family Offline-to-Online RL}
\begin{wrapfigure}{r}{0.5\textwidth}
\vspace{-50pt}
\begin{center}
\includegraphics[width=.5\textwidth]{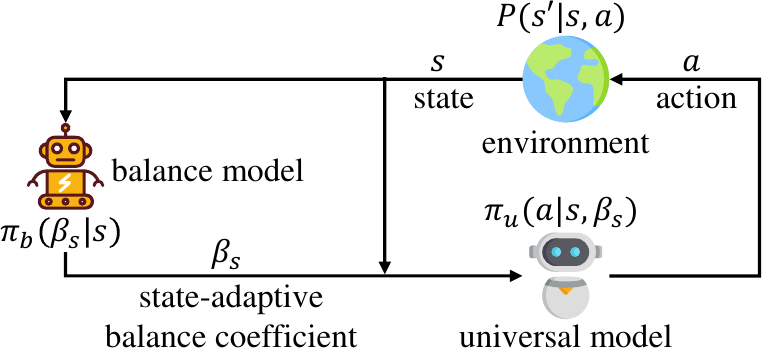}
    \vspace{-10pt}
    \caption{FamO2O's \textbf{inference} process. 
    For each state $\s$, the balance model $\pi_b$ computes a state-adaptive balance coefficient $\beta_\s$.
    Based on $\s$ and $\beta_\s$, the universal model $\pi_u$ outputs an action $\act$.}
    \label{fig: framework}
\end{center}
    \vspace{-15pt}
\end{wrapfigure}

\cref{sec: rationality of state-adaptive balance coefficients} theoretically proves that finer-grained policy constraints enhance performance upper bounds, necessitating state-adaptive balance coefficients. 
Accordingly, we introduce FamO2O, 
a framework adaptively assigning a balance~coefficient to each state, easily implemented over various offline-to-online algorithms like \cite{nair2020awac, iql, kumar2020cql}, hereafter called the ``base algorithm''.

Essentially, FamO2O trains a policy family with varying balance coefficients during offline pre-training. During online fine-tuning, FamO2O identifies the appropriate policy, corresponding to the suitable balance coefficient, for each state from this policy family.
In this section, we present FamO2O using the AWR algorithm framework (\cref{eq: training objective}).
FamO2O's compatibility with non-AWR algorithms is discussed in \cref{sec: FamO2O extension}.

As shown in \cref{fig: framework}, FamO2O's policy consists of two components: a \emph{universal model} $\pi_u$ and a \emph{balance model} $\pi_b$.
Denote the space of balance coefficients as $\mathcal{B}$. 
For every state $\s$, the balance model\footnote{Despite $\pi_u$ and $\pi_b$ being stochastic models, we notate them as functions using ``$\mapsto$" hereafter for brevity.} $\pi_b: \mathcal{S}\mapsto\mathcal{B}$ figures out a suitable balance coefficient $\beta_\s$;
based on the state $\s$ and state-related balance coefficient $\beta_\s$, the universal model $\pi_u: \mathcal{S}\times \mathcal{B} \mapsto \mathcal{A}$ outputs an action.
The balance coefficient $\beta_\s$ is to control the conservative/radical degree of the universal model $\pi_u$ in dealing with the state $\s$.

\subsection{Learning universal model}\label{sec: learning universal model}
We initially address training the universal model $\pi_u$, aimed at learning a policy family with varying balances between policy improvement and constraint. The formal optimization target of $\pi_u$ is:
\begin{equation}
\label{eq: the optimization target of the universal model}
     \pi_u^{k+1} = \argmax{\pi_u}\mathbb{E}_{(\s, \act)\sim \mathcal{D}} \Big[ \exp(\beta_\s (Q^{k}(\s, \act)-V^{k}(\s))) \log \pi_u(\act|\s, {\color{blue}\beta_\s}) \Big].
\end{equation} 
$Q^{k}$ and $V^k$, detailed in \cref{sec: learning value functions}, represent $Q$ and $V$ functions at iteration $k$. \cref{eq: the optimization target of the universal model} echoes \cref{eq: adaptive training objective}, but the policy also takes a balance coefficient $\beta_\s$ as input (highlighted in {\color{blue} blue}). In the offline pre-training phase, $\beta_\s$ is randomly sampled from balance coefficient space $\mathcal{B}$. This encourages $\pi_u$ to learn varied strategies. During online fine-tuning, $\beta_\s$ is set by balance model $\pi_b$ before input to universal model $\pi_u$, which prompts cooperation between $\pi_u$ and $\pi_b$.

\subsection{Learning balance model}\label{sec: learning balance model}
Next, we outline how the balance model $\pi_b$ chooses an appropriate policy for each state from the policy family trained by the universal model $\pi_u$. As indicated in \cref{sec: learning balance model}, every $\beta_{\s} \in \mathcal{B}$ corresponds to a unique policy. Consequently, to select the optimal policy, $\pi_b$ needs to determine the appropriate balance coefficient $\beta_\s$ for each state $\s$. Given this rationale, the update rule for $\pi_b$ is:
%
\begin{equation}
    \pi_b^{k+1} = \argmax{\pi_b} \mathbb{E}_{(\s, \act)\sim \mathcal{D}}\Big[Q^{k}(\s, \underbrace{\pi_u^{k+1}(\s, \overbrace{\pi_b(\s)}^{\mathclap{\text{balance coefficient }\beta_\s}}}_{\text{action}})\Big].
    \label{eq: the optimization target of the balance model}
\end{equation}  

Here, $\pi_u^{k+1}$ represents the updated universal model in \cref{eq: the optimization target of the universal model}. Intuitively, \cref{eq: the optimization target of the balance model} aims to find a $\pi_b$ that maximizes $Q^k$ value by translating balance coefficients into actions with $\pi_u^{k+1}$. 
This design is grounded in the understanding that the $Q$ value serves as an estimate of future return, which is our ultimate goal of striking a balance between policy improvement and constraint.
Concerns may arise about $Q^k$'s extrapolation error in \cref{eq: the optimization target of the balance model} potentially misguiding $\pi_b$'s update. Empirical evidence suggests this is less of an issue if we avoid extremely radical values in the balance coefficient space $\mathcal{B}$.
Following the update rule in \cref{eq: the optimization target of the balance model}, $\pi_b$ effectively assigns balance coefficients to states, demonstrated in \cref{sec: toy example}.

\subsection{Learning value functions}\label{sec: learning value functions}
Furthermore, we explain the value functions update. As per \crefrange{eq: the optimization target of the universal model}{eq: the optimization target of the balance model}, a single set of $Q$ and $V$ functions evaluate both $\pi_u$ and $\pi_b$. This is due to $\pi_b: \mathcal{S} \mapsto \mathcal{B}$ and $\pi_u: \mathcal{S}\times \mathcal{B} \mapsto \mathcal{A}$ collectively forming a standard RL policy $\pi_u(\cdot, \pi_b(\cdot)): \mathcal{S}\mapsto \mathcal{A}$. Hence, the value functions update mirrors that in the base algorithm, simply replacing the original policy with $\pi_u(\cdot, \pi_b(\cdot))$.

Finally, we offer a pseudo-code of FamO2O's training process in \cref{sec: pseudo-code}.
%

%% file: sections/4_experiments.tex
In this section, we substantiate the efficacy of FamO2O through empirical validation. We commence by showcasing its state-of-the-art performance on the D4RL benchmark~\cite{fu2020d4rl} with IQL~\cite{iql} in \cref{sec: comparison on benchmark}. We then evidence its performance improvement's generalizability and statistical significance in \cref{sec: performance improvement bought by FamO2O}. Moreover, \cref{sec: exp on other algorithms} reveals FamO2O's compatibility with non-AWR-style algorithms like CQL~\cite{kumar2020cql}, yielding significant performance enhancement. Lastly, we reserve detailed FamO2O analyses for \cref{sec: discussion} and \cref{appendix: more results}.
For more information on implementation details, please refer to \cref{appendix: FamO2O implementation details}.

\begin{wrapfigure}{r}{0.57\textwidth}
\vspace{-2pt}
    \centering
\includegraphics[width=0.48\textwidth]{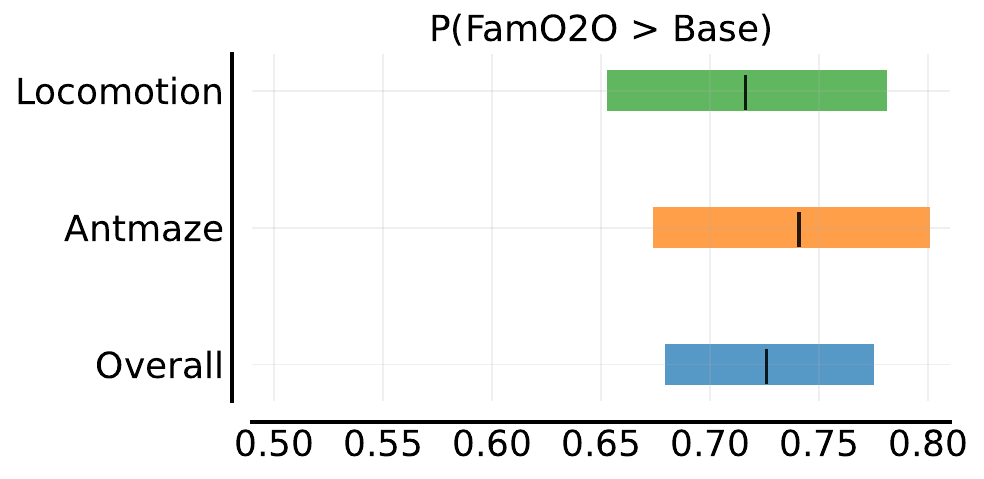}
  \caption{\textbf{FamO2O's improvement over base algorithms~\cite{nair2020awac, iql}.} For D4RL Locomotion, AntMaze~\cite{fu2020d4rl}, and overall, FamO2O shows significant and meaningful performance gains, meeting Neyman-Pearson criteria~\cite{rliable}.}
    \label{fig:prob-of-improvement}
    \vspace{-10pt}
\end{wrapfigure}

\paragraph{Datasets} 
Our method is validated on two D4RL~\cite{fu2020d4rl} benchmarks: Locomotion and AntMaze. Locomotion includes diverse environment datasets collected by varying quality policies. We utilize IQL~\cite{iql} settings, assessing algorithms on hopper, halfcheetah, and walker2d environment datasets, each with three quality levels. AntMaze tasks involve guiding an ant-like robot in mazes of three sizes (umaze, medium, large), each with two different goal location datasets. The evaluation environments are listed in \cref{tab: family-vs-base}'s first column.

\begin{figure}[!t]
    \centering    
    \includegraphics[width=.95\columnwidth]{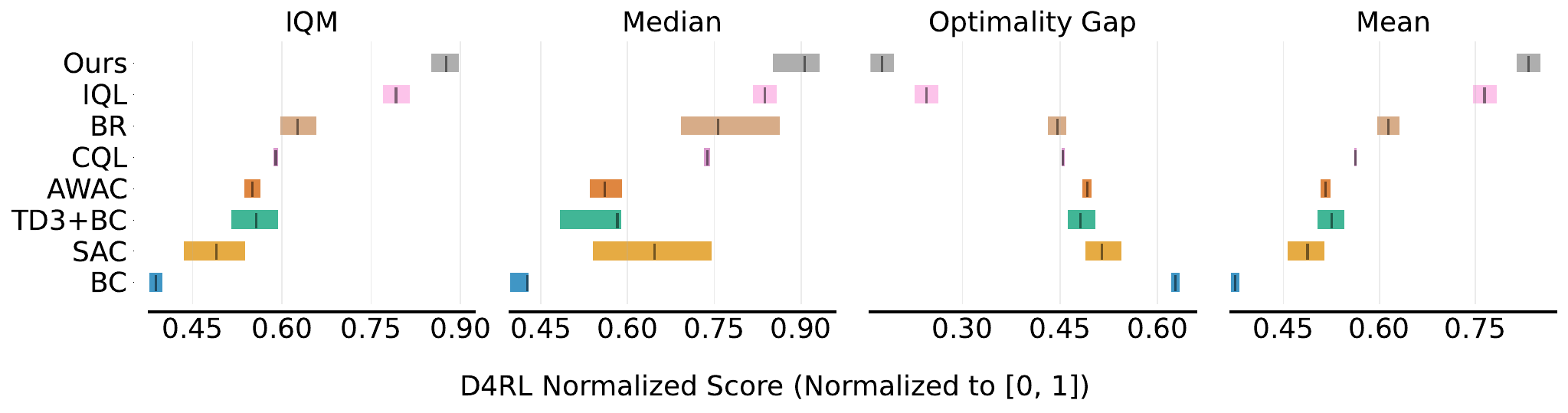}
    \vspace{-5pt}
    \caption{\textbf{Comparisons between our FamO2O against various competitors} on D4RL normalized scores~\cite{fu2020d4rl}.
    All methods are tested on D4RL Locomotion and AntMaze for 6 random seeds. 
    FamO2O achieves state-of-the-art performance by a statistically significant margin among all the competitors in offline-to-online RL (\textit{i.e.} IQL~\cite{iql}, Balaned Replay (BR)~\cite{lee2022pessimistic}, CQL~\cite{kumar2020cql}, AWAC~\cite{nair2020awac}, and TD3+BC~\cite{td3+bc}), online RL (\textit{i.e.} SAC~\cite{sac}), and behavior cloning~(BC).
    }
    \vspace{-5pt}
    \label{fig: compare-with-baseline}
\end{figure}

{
\setlength\fboxsep{2pt}
\begin{table*}[!t]
\begin{threeparttable}
  \small
  \caption{\textbf{Enhanced performance achieved by FamO2O after online fine-tuning.}
We evaluate the D4RL normalized score~\cite{fu2020d4rl} of standard base algorithms (including AWAC~\cite{nair2020awac} and IQL~\cite{iql}, denoted as "Base") in comparison to the base algorithms augmented with FamO2O (referred to as "Ours").
All results are assessed across 6 random seeds.
The superior offline-to-online scores are highlighted in blue.
\textbf{FamO2O consistently delivers statistically significant performance enhancements across different algorithms and task sets.}
  }
  \label{tab: family-vs-base}
  \centering
  \renewcommand{\tabcolsep}{1.7pt}
  \begin{tabular}{l|cc|cc||cc}
    \specialrule{0.12em}{0pt}{0pt}
	\multirow{2}{*}{Dataset\tnote{1}}
	& \multicolumn{2}{c|}{AWAC~\cite{nair2020awac}} 
	& \multicolumn{2}{c||}{IQL~\cite{iql}} 
	& \multicolumn{2}{c}{\textbf{Avg.}}
	\\ 
	\cline{2-7}
	& Base & Ours
	& Base & Ours
	& Base & Ours
	\\ \hline
    hopper-mr-v2
    & 56.0 & \colorbox{myblue}{86.8} & 91.0 & \colorbox{myblue}{97.6} & 73.5 & \colorbox{myblue}{92.2}
    \\
    hopper-m-v2
    & 54.1  & \colorbox{myblue}{75.0} & 65.4 & \colorbox{myblue}{90.7} & 59.7 & \colorbox{myblue}{82.8} 
    \\
    hopper-me-v2
    & \colorbox{myblue}{97.7} & 92.9 & 76.5 & \colorbox{myblue}{87.3} & 87.1 & \colorbox{myblue}{90.1} 
    \\
    halfcheetah-mr-v2
    & 43.9 & \colorbox{myblue}{49.0} & \colorbox{myblue}{53.7} & 53.1 & 48.8 & \colorbox{myblue}{51.0} 
    \\
    halfcheetah-m-v2 
    & 44.8 & \colorbox{myblue}{47.6} & 52.5 & \colorbox{myblue}{59.2} & 48.7 & \colorbox{myblue}{53.4} 
    \\  
    halfcheetah-me-v2 
    & \colorbox{myblue}{91.0} & 90.6 & 92.8 & \colorbox{myblue}{93.1} & \colorbox{myblue}{91.9} & 91.8
    \\ 
    walker2d-mr-v2 
    & 72.8 & \colorbox{myblue}{84.4} & 90.1 & \colorbox{myblue}{92.9} & 81.5 & \colorbox{myblue}{88.6}
    \\ 
    walker2d-m-v2
    & 79.0 & \colorbox{myblue}{80.0} & 83.8 & \colorbox{myblue}{85.5} & 81.4 & \colorbox{myblue}{82.8}
    \\ 
    walker2d-me-v2 
    & \colorbox{myblue}{109.3} & 108.5 & 112.6 & \colorbox{myblue}{112.7} & \colorbox{myblue}{110.9} & 110.6
    \\  
    \hline
    \textbf{locomotion total} 
    & 648.4 & \colorbox{myblue}{714.9} & 718.3 & \colorbox{myblue}{772.0} & 683.4 & \colorbox{myblue}{743.4}
    \\
    \textbf{95\% CIs} 
    & 640.5\textasciitilde656.8 & 667.3\textasciitilde761.4 & 702.5\textasciitilde733.5 & 753.5\textasciitilde788.5 & 674.6\textasciitilde692.0 & 732.1\textasciitilde754.2
    \\ 
    \hline
    umaze-v0
    & 64.0 & \colorbox{myblue}{96.9} & 96.5 & \colorbox{myblue}{96.7} & 80.4 & \colorbox{myblue}{96.8}
    \\ 
    umaze-diverse-v0
    & 60.4 & \colorbox{myblue}{90.5} & 37.8 & \colorbox{myblue}{70.8} & 66.2 & \colorbox{myblue}{80.6}
    \\ 
    medium-diverse-v0
    & 0.2 & \colorbox{myblue}{22.2} & 92.8 & \colorbox{myblue}{93.0} & 45.2 & \colorbox{myblue}{57.6}
    \\ 
    medium-play-v0
    & 0.0 &  \colorbox{myblue}{34.2} & 91.5 & \colorbox{myblue}{93.0} & 45.2 & \colorbox{myblue}{63.6}
    \\ 
    large-diverse-v0
    &  0.0 & 0.0 & 57.5 & \colorbox{myblue}{64.2} & 24.7 & \colorbox{myblue}{32.1}
    \\ 
    large-play-v0
    & 0.0 &  0.0 & 52.5 & \colorbox{myblue}{60.7} & 21.4 & \colorbox{myblue}{30.3}
    \\ 
    \hline
    \textbf{antmaze total} 
    & 124.7 & \colorbox{myblue}{243.7} & 428.7 & \colorbox{myblue}{478.3} & 283.1 & \colorbox{myblue}{361.1}
    \\
    \textbf{95\% CIs}
    & 116.5\textasciitilde132.6 & 226.2\textasciitilde259.9 & 406.7\textasciitilde452.7 & 456.7\textasciitilde498.7 & 274.1\textasciitilde291.1 & 347.2\textasciitilde374.3
    \\
    \hline \hline
    \textbf{total} 
    & 773.0 & \colorbox{myblue}{958.6} & 1146.9 & \colorbox{myblue}{1250.3} & 960.0 & \colorbox{myblue}{1104.5}
    \\
    \textbf{95\% CIs} & 761.5\textasciitilde784.6 & 936.8\textasciitilde979.6 & 1119.5\textasciitilde1175.1 & 1221.9\textasciitilde1277.0 & 911.3\textasciitilde1008.8 & 1063.5\textasciitilde1145.4
    \\
    \specialrule{0.12em}{0pt}{0pt}
  \end{tabular}
  \begin{tablenotes}
            \item[1] mr: medium-replay, m: medium, me: medium-expert.
        \end{tablenotes}
\end{threeparttable}
\vspace{-22pt}
\end{table*}
}

\vspace{-10pt}
\paragraph{Metrics} Considering RL's inherent variability, we adopt robust evaluation methods per rliable~\cite{rliable}. Besides conventional Medium and Mean, we integrate IQM and Optimality Gap metrics for broader assessment. We also employ rliable's probability of improvement metric for gauging the likelihood of our method outperforming others. We confirm our performance enhancement's statistical significance using 95\% Confidence Intervals (CIs).




\subsection{Benchmark Comparison}
\label{sec: comparison on benchmark}

FamO2O's state-of-the-art performance is demonstrated by implementing it over IQL~\cite{iql} and comparing with baseline methods.

\paragraph{Baselines} We benchmark FamO2O against:
\begin{inlinelist}
    \item \textbf{offline-to-online RL}, including IQL~\cite{iql}, Balanced Replay (BR)~\cite{lee2022pessimistic}, CQL~\cite{kumar2020cql}, AWAC~\cite{nair2020awac}, and TD3+BC~\cite{td3+bc}.
    For IQL, BR, and AWAC, official implementations are used. 
    In the case of CQL and TD3+BC, we implement online fine-tuning based on the author-provided offline pre-training codes, following the procedures in IQL and AWAC;
    \item \textbf{online RL method}, SAC~\cite{sac}, to highlight offline pre-training's efficacy;
    \item \textbf{behavior cloning (BC)}, which is implemented by maximizing the log-likelihood of the samples in the offline dataset.
\end{inlinelist}
For SAC and BC, we utilize the implementations of CQL.
Further details are in \cref{appendix: baseline implementation}.

\begin{figure}[!t]
\centering
\begin{minipage}[t]{0.39\textwidth}
\centering
\includegraphics[width=0.95\columnwidth]{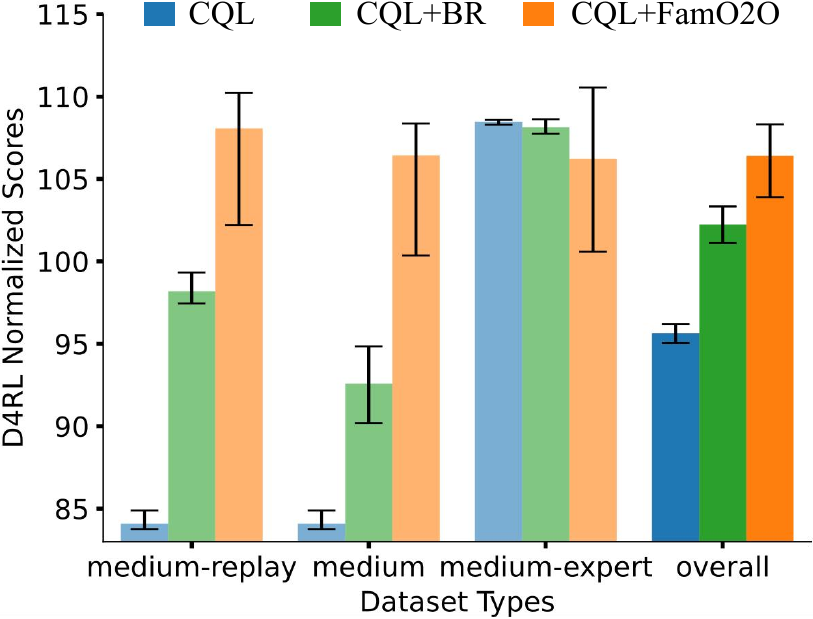}
\vspace{-5pt}
    \caption{\textbf{IQM scores of FamO2O's implementation on top of CQL~\cite{kumar2020cql}.} Demonstrating a statistically significant superiority over vanilla CQL~\cite{kumar2020cql} and CQL+BR~\cite{lee2022pessimistic}, FamO2O affirms its \textbf{adaptability to non-AWR algorithms}.}
    \label{fig: FamO2O on CQL}
\end{minipage}
\ \ \ 
\begin{minipage}[t]{0.58\textwidth}
\centering
    \includegraphics[width=1.0\columnwidth]{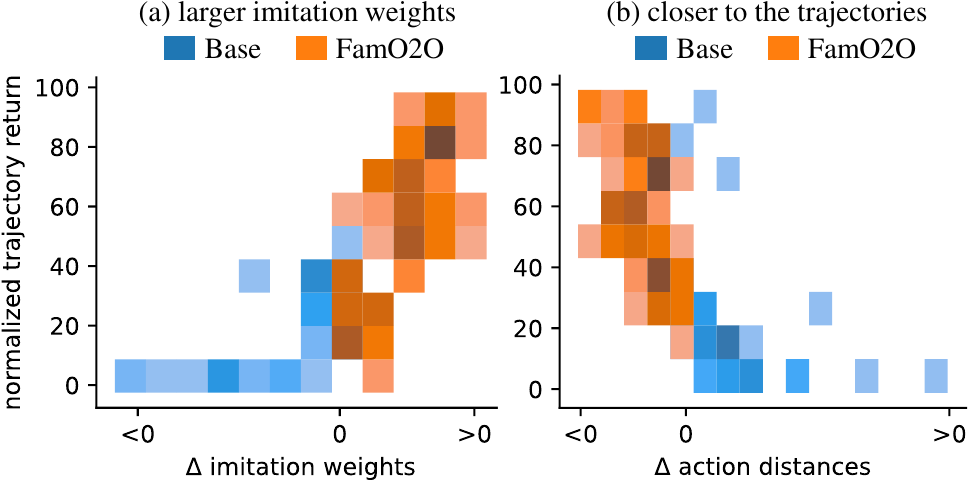}
    \vspace{-15pt}
    \caption{Comparing FamO2O with the base algorithm~\cite{iql} on \textbf{(a) imitation weights} and \textbf{(b) action distances}. $\Delta$ indicates the difference between FamO2O's metrics and the base's. Generally, FamO2O \textbf{emphasizes high-quality data imitation} and \textbf{aligns more closely with high-quality trajectories} compared to the base algorithm.}
    \label{fig:effect-of-state-adaptive-policy-constraint}
\end{minipage}
\vspace{-8pt}
\end{figure}

\vspace{-10pt}
\paragraph{Comparison}
As shown in \cref{fig: compare-with-baseline}, FamO2O outperforms competitors across all metrics (IQM, Medium, Optimality Gap, and Mean). Specifically, for IQM, Optimality Gap, and Mean, FamO2O's 95\% CIs don't overlap with the competitors'. Even for Medium, all baseline expectations fall below the lower limit of FamO2O's 95\% CIs. The results underscore the significant edge of our state-adaptive policy constraint mechanism over competing methods.

\subsection{Analyzing FamO2O's Performance Enhancement}
\label{sec: performance improvement bought by FamO2O}

Though FamO2O demonstrated superior performance on the D4RL benchmark in \cref{sec: comparison on benchmark}, it's vital to discern the actual contributions of FamO2O from its base policy, IQL. Therefore, we address two key questions:
\begin{inlinelist}
\item Does FamO2O consistently enhance other offline-to-online RL algorithms?
\item Is the performance boost by FamO2O statistically significant given RL's inherent variability?
\end{inlinelist}

\vspace{-10pt}
\paragraph{Setup}
We apply FamO2O to AWAC and IQL. AWAC~\cite{nair2020awac} is one of the most famous offline-to-online algorithms, and IQL~\cite{iql} is a recently proposed method that achieves great performance on D4RL~\cite{fu2020d4rl}. We use the authors' codes, maintaining the same hyperparameters for a fair comparison. Further details are in \cref{appendix: FamO2O implementation details}.

\vspace{-10pt}
\paragraph{Comparison}
\Cref{tab: family-vs-base} shows AWAC's and IQL's performances w/ and w/o FamO2O. FamO2O~generally enhances performance by a statistically significant margin across most datasets, regardless of the base algorithm, highlighting its versatility. 
Even on complex datasets where~AWAC barely~succeeds, \textit{e.g.}, AntMaze medium-diverse and medium-play, FamO2O still achieves commendable~performance.

Pursuing rliable's recommendation~\cite{rliable}, we evaluated FamO2O's statistical significance by calculating average probabilities of improvement against base policies (\cref{fig:prob-of-improvement}). In all three cases (Locomotion, AntMaze, and Overall), the lower CI bounds exceed 50\%, denoting the statistical significance of FamO2O's improvement. Specifically, the upper CI on Locomotion surpasses 75\%, demonstrating statistical meaning as per the Neyman-Pearson criterion.

\subsection{Versatility of FamO2O with Non-AWR Algorithms}\label{sec: exp on other algorithms}
To demonstrate FamO2O's versatility beyond AWR-based algorithms, we extended it to CQL~\cite{kumar2020cql} in addition to AWAC~\cite{nair2020awac} and IQL~\cite{iql}. The implementation specifics are in \cref{sec: FamO2O extension}. As \cref{fig: FamO2O on CQL} reveals, FamO2O significantly outperforms CQL. Even when compared to Balance Replay (BR)~\cite{lee2022pessimistic}, an offline-to-online method designed specifically for CQL, FamO2O still shows statistically significant superior performance. These results highlight FamO2O's adaptability to non-AWR algorithms.

%% file: sections/5_discussion.tex
In this section, we further provide some in-depth studies on FamO2O, including visualization (\cref{sec: toy example}) and quantitative analyses (\crefrange{sec:effect-of-state-adaptive-policy-constraint}{sec:ablation-on-random-betas}).
More analyses are deferred to \cref{appendix: more results}.

\subsection{Does FamO2O really have state-wise adaptivity?}\label{sec: toy example}

\begin{wrapfigure}{r}{0.523\textwidth}
    \centering
    \vspace{-10pt}    
    \includegraphics[width=0.523\textwidth]{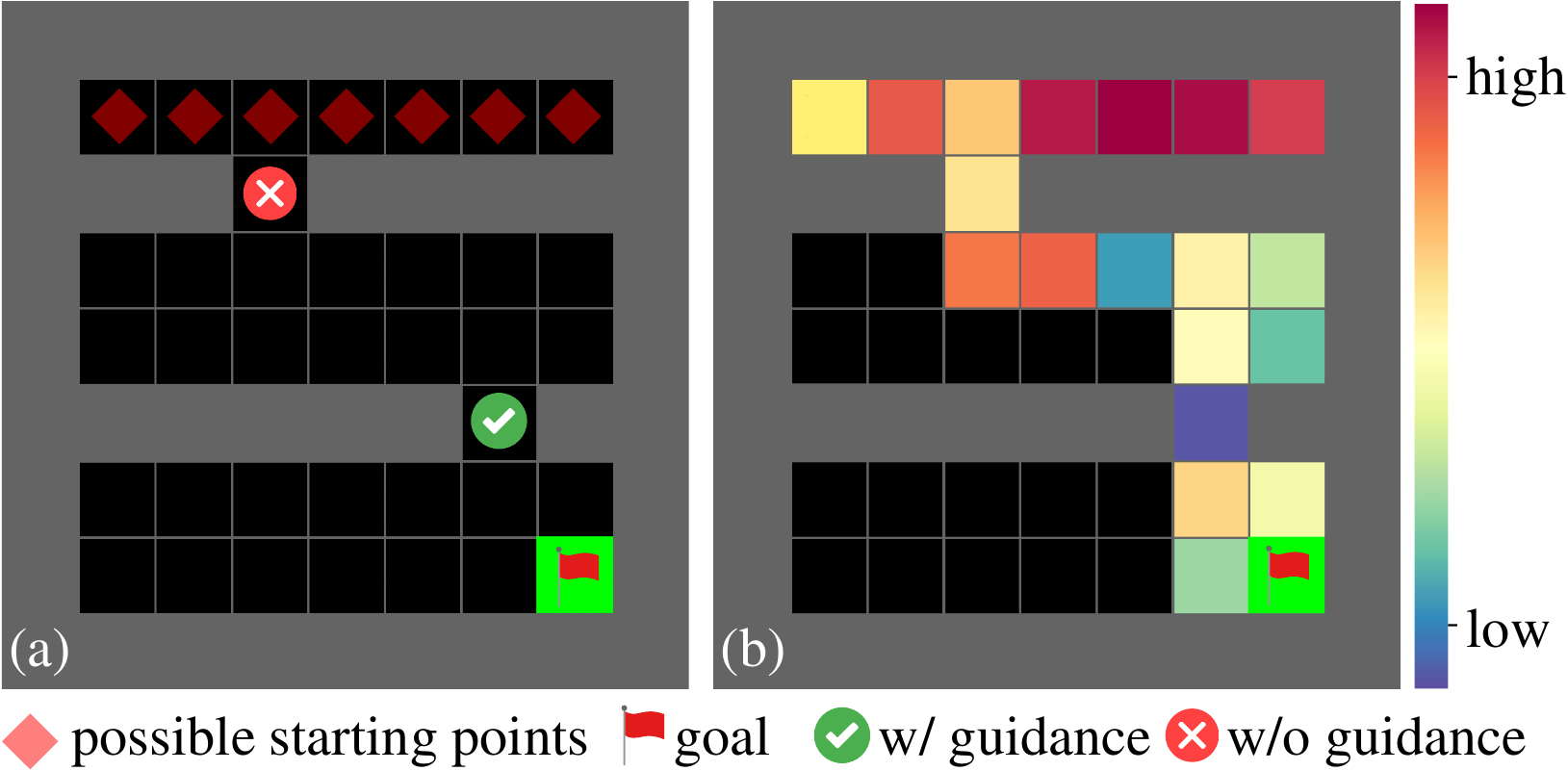}
    \vspace{-15pt}
    \caption{
\textbf{State-wise adaptivity visualization} in a simple maze environment. (a) Higher data quality at the crossing point in the 5th row compared to the 2nd row. (b) Colors denote different balance coefficient values at traversed cells during inference. FamO2O typically displays conservative (or radical) behavior at cells with high-quality (or low-quality) data.}
    \label{fig:toy-example}
    \vspace{-5pt}
\end{wrapfigure}

\vspace{-5pt}
Here, we design a simple maze environment to visualize the state-wise adaptivity of FamO2O.
As shown in \cref{fig:toy-example}(a), the agent starts at a random cell in the top row and is encouraged to reach the goal at the bottom right corner through two crossing points. During data collection, guidance is provided when the agent passes through the lower crossing point, but no guidance for the upper crossing point.
To elaborate, the guidance refers to compelling the agent to adhere to the route and direction that yields the shortest path to the goal. Without it, the agent moves randomly.
It can be observed in \cref{fig:toy-example}(b) that the agent generally outputs lower balance coefficients for the states with high-quality samples (\textit{i.e.}, those derived from the agent's movement with guidance) while generating higher balance coefficients for the states with low-quality data (\textit{i.e.}, data gathered when the agent moves without guidance).
This result shows FamO2O's state-wise adaptivity in choosing proper balance coefficients according to the data quality related to the current state.

\subsection{What is the effect of state-adaptive balances?}\label{sec:effect-of-state-adaptive-policy-constraint}
\vspace{-5pt}
In this section, we explore the impact of state-adaptive improvement-constraint balance. In order to encompass data of varying quality levels, we assess FamO2O using the medium-replay datasets of D4RL Locomotion~\cite{fu2020d4rl}. Our analysis focuses on two metrics: \textit{imitation weights} and \textit{action distances}.

\emph{Imitation weights} are defined in \cref{eq: training objective}, with larger (or smaller) values prompting the agent to align more closely (or less) with the replay buffer $\mathcal{D}$'s behavior. \emph{Action distance}, delineated in \cref{eq: action distance}, quantifies the discrepancy in action selection between a policy $\pi$ and a trajectory $\tau$:
%
\begin{equation}
    d_{\text{action}}^{\pi, \tau} = \mathbb{E}_{(\s, \act) \sim \tau}\big[ \| \argmax{\act^\prime}\ \pi(\act^\prime|\s) - \act\|^2_2 \big].
    \label{eq: action distance}
\end{equation}
%
Here, a lower (or higher) action distance $d_{\text{action}}^{\pi, \tau}$ signifies a greater (or lesser) behavioral similarity between the policy $\pi$ and the trajectory $\tau$.

We evaluate IQL with and without FamO2O (`Base' and 'FamO2O') regarding imitation weights and action distances, depicted in \cref{fig:effect-of-state-adaptive-policy-constraint}. \cref{fig:effect-of-state-adaptive-policy-constraint}(a) computes the average imitation weight difference (AIWD) per trajectory in the offline dataset. AIWD indicates the mean imitation weight difference between FamO2O and the base algorithm for each $(s, a)$ pair within a trajectory. \cref{fig:effect-of-state-adaptive-policy-constraint}(b) likewise determines an average action distance difference per offline dataset trajectory.

\cref{fig:effect-of-state-adaptive-policy-constraint}(a) reveals that FamO2O typically shows higher imitation weights than the base algorithm for high-return trajectories. \cref{fig:effect-of-state-adaptive-policy-constraint}(b) indicates that FamO2O aligns more with high-quality trajectories and less with low-quality ones than the base algorithm. These results highlight the state-adaptive balance's role in promoting emulation of high-quality behavior and avoidance of low-quality behavior.

\subsection{State-adaptive balances \textit{vs.} fixed balances?}
\label{sec:ablation-on-fixed-betas}

\begin{figure}[!t]
\centering
\begin{minipage}[t]{0.56\textwidth}
\centering
\includegraphics[width=1\textwidth]{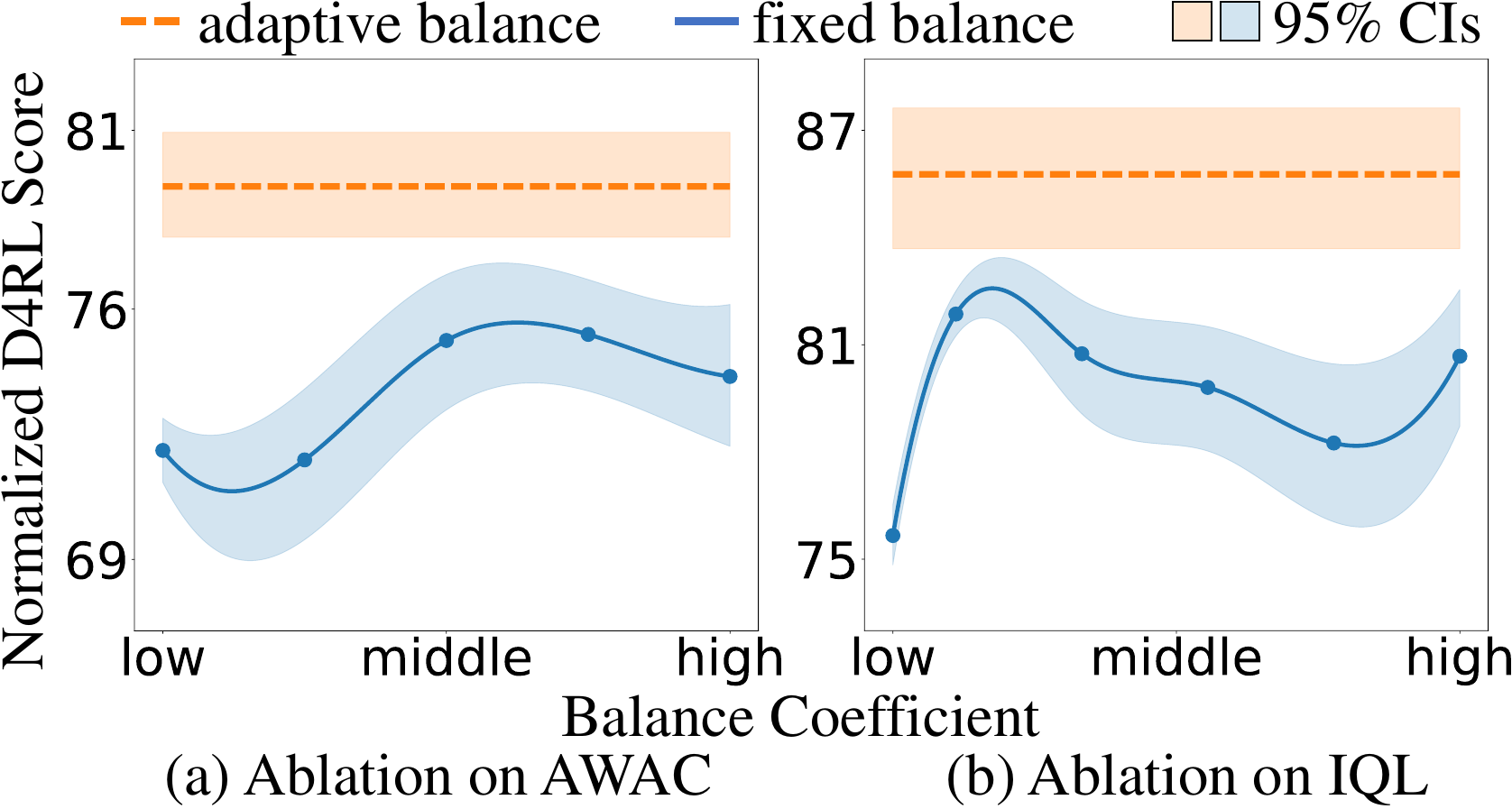}
    \vspace{-15pt}
    \caption{
\textbf{Comparing adaptive and fixed balance coefficients} on (a) AWAC~\cite{nair2020awac} and (b) IQL~\cite{iql} using D4RL Locomotion datasets~\cite{fu2020d4rl}. Adaptive coefficients consistently outperform fixed ones with distinct 95\% CIs.
    }
    \label{fig: fixed-constraint-degree}
    \vspace{-15pt}
\end{minipage} \ \ \ \ \ 
\begin{minipage}[t]{0.37\textwidth}
    \centering
    \includegraphics[width=1\textwidth]{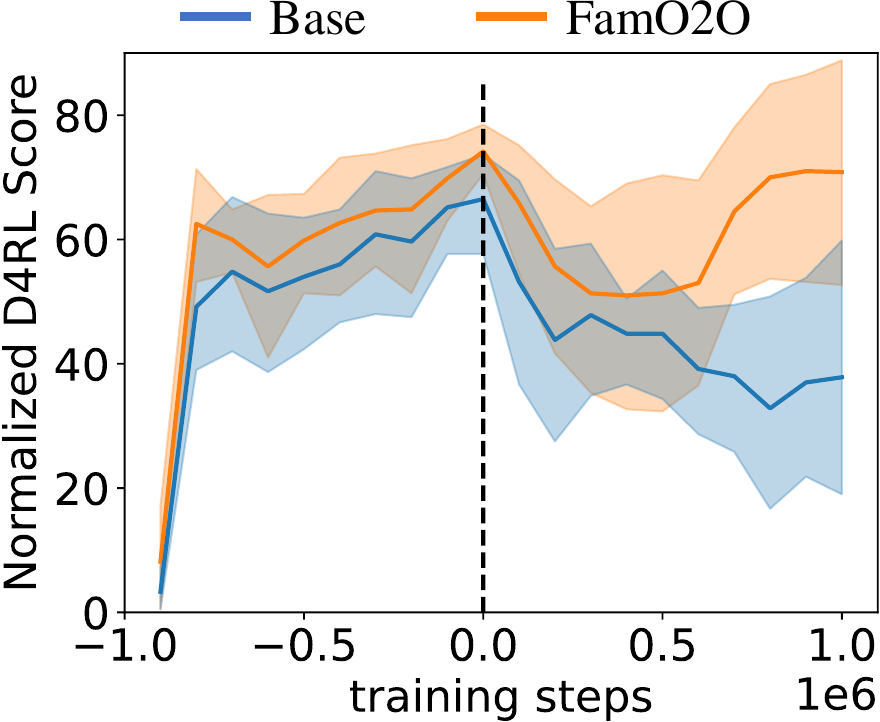}
    \vspace{-15pt}
    \caption{FamO2O \textbf{alleviates performance drop} due to \textbf{distributional shift} during the shift from offline pre-training to online~fine-tuning.}
    \label{fig: distributional shift handling}
\end{minipage}
\vspace{-5pt}
\end{figure}

\vspace{-5pt}
To prove that our adaptive balance coefficients are better than traditional fixed balance coefficients, we compare FamO2O against the base algorithms with different balance coefficients as hyperparameters.
As shown in \cref{fig: fixed-constraint-degree}, on both AWAC~\cite{nair2020awac} and IQL~\cite{iql}, our adaptive balance coefficients outperform all the fixed balance coefficients.
Significantly, the 95\% CIs of adaptive and fixed balance coefficients have no overlap.
These comparison results indicate that our adaptive balance coefficient approach surpasses the fixed balance coefficient method by a statistically significant margin.

\subsection{Does FamO2O's efficacy rely on varied data qualities?}

\begin{wrapfigure}{r}{0.58\columnwidth}
    \centering
    \vspace{-10pt}
    \includegraphics[width=0.58\columnwidth]{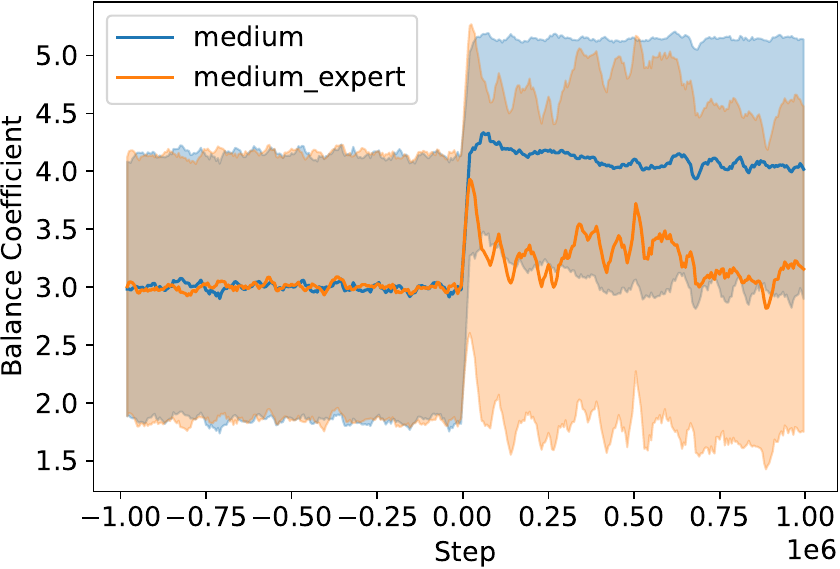}
    \vspace{-15pt}
    \caption{Balance coefficients' mean and std for IQL w/ FamO2O on D4RL HalfCheetah.}
    \label{fig: means and std of balance coefficients}
    \vspace{-5pt}
\end{wrapfigure}

It's worth noting that our method's efficacy doesn't rely on varied data qualities.
Table~\ref{tab: family-vs-base} clearly demonstrates that FamO2O surpasses the base algorithms in performance across all the medium datasets, namely hopper-medium-v2, halfcheetah-medium-v2, and walker2d-medium-v2, which all maintain a relatively consistent data quality.
We claim that FamO2O can determine the suitable conservative/radical balance for each state in online scenarios based on the data quality in the collected dataset. 
If the dataset is diverse in quality, the balances will be diverse; if the quality is consistent, the balances will be correspondingly consistent.
The above claim is supported by Figure~\ref{fig: means and std of balance coefficients}, which indicates that 
\begin{inlinelist}
    \item in datasets with more (or less) diverse data qualities, i.e., medium-expert (or medium), the balance coefficients are more (or less) diverse, with corresponding larger (or smaller) standard deviations;
    \item with higher (or lower) quality datasets, the balance coefficients are averagely lower (or higher), leading to a more conservative (or radical) policy.
\end{inlinelist}

\subsection{Does FamO2O mitigate the performance drop stemming from the distributional shift?}
\vspace{-5pt}
Despite FamO2O's primary objective not being direct distributional shift handling, its state-adaptive improvement-constraint balances prove beneficial in mitigating performance degradation during the offline pre-training to online fine-tuning transition, attributed to the distributional shift.
Existing offline-to-online RL algorithms~\cite{iql, nair2020awac} already incorporate mechanisms to counter distributional shift, hence significant performance drops are infrequent.
\cref{fig: distributional shift handling} illustrates the training curves of IQL and IQL+FamO2O across both offline pre-training (negative steps) and online fine-tuning (positive steps) on antmaze-umaze-diverse, on which IQL exhibits the most significant decline in performance when transitioning from offline pre-training to online fine-tuning. 
As evidenced, while IQL+FamO2O initially experiences a performance drop akin to IQL during online fine-tuning, it recovers rapidly and attains high performance, in stark contrast to IQL's sustained performance decline throughout the fine-tuning stage.

\subsection{Balance model \textit{vs.} random selector?}
\label{sec:ablation-on-random-betas}
\vspace{-5pt}

\begin{wraptable}{r}{0.52\textwidth}
\begin{threeparttable}
  \vspace{-28pt}
  \small
  \caption{\textbf{Ablation study on the balance model.} The absence of the balance model results in decreased performance of FamO2O on nearly all D4RL Locomotion datasets~\cite{fu2020d4rl}. The model's impact is statistically significant and meaningful as per the Neyman-Person criterion.
  }
  \centering
    \renewcommand{\tabcolsep}{2.7pt}
  \begin{tabular}{lcc}
    \specialrule{0.12em}{0pt}{0pt}
	Dataset 
	& FamO2O 
	& random-FamO2O
	\\ \hline
    hopper-mr-v2
    &  \colorbox{myblue}{97.64} 
    & 80.87
    \\
    hopper-m-v2
    & \colorbox{myblue}{90.65}
    & 86.37
    \\
    hopper-me-v2
    & \colorbox{myblue}{87.28}
    & 77.08
    \\
    halfcheetah-mr-v2
    & 53.07  
    & \colorbox{myblue}{53.75}
    \\ 
    halfcheetah-m-v2
    & \colorbox{myblue}{59.15}
    & 53.15
    \\ 
    halfcheetah-me-v2
    &  \colorbox{myblue}{93.10}
    & 92.72
    \\ 
    walker2d-mr-v2
    & \colorbox{myblue}{92.85}
    & 91.38
    \\ 
    walker2d-m-v2
    & \colorbox{myblue}{85.50}
    & 84.84
    \\ 
    walker2d-me-v2
    & \colorbox{myblue}{112.72}
    & 110.54
    \\ 
    \hline
    \textbf{total} 
    & \colorbox{myblue}{\textbf{771.96}}
    & \textbf{730.70} 
    \\
    \textbf{prob. of improvement} & \multicolumn{2}{c}{\textbf{0.70}\ \ (95\% CIs: \textbf{0.59}$\sim$\textbf{0.79})} \\
    \specialrule{0.12em}{0pt}{0pt}
  \end{tabular}
  \label{tab:random-FamO2O}
  \vspace{-5pt}
  \end{threeparttable}
\end{wraptable}

To validate the effect of the balance model $\pi_b$ in choosing balance coefficients, we present a FamO2O variant, denoted as \textit{random-FamO2O}, with the balance model replaced with a random balance coefficient selector.
Other training settings keep the same for FamO2O and random-FamO2O.
\cref{tab:random-FamO2O} shows the improvement percentages and probability of improvement of FamO2O against random-FamO2O.
As we can observe, FamO2O outperforms random-FamO2O on almost all the datasets of D4RL Locomotion~\cite{fu2020d4rl}.
Furthermore, the lower CI of the probability of improvement is much higher than 50\%, and the upper CI exceeds 75\%.
This indicates that the effect of the balance model is not only statistically significant but also statistically meaningful as per the Neyman-Person statistical testing criterion.

%% file: sections/7_conclusion.tex
\vspace{-5pt}
This work underscores the significance of state-adaptive improvement-constraint balances in offline-to-online RL. We establish a theoretical framework demonstrating the advantages of these state-adaptive balance coefficients for enhancing policy performance. Leveraging this analysis, we put forth Family Offline-to-Online RL (FamO2O), a versatile framework that equips existing offline-to-online RL algorithms with the ability to discern appropriate balance coefficients for each state.

Our experimental results, garnered from a variety of offline-to-online RL algorithms, offer substantial evidence of FamO2O's ability to significantly improve performance, attaining leading scores on the D4RL benchmark. In addition, we shed light on FamO2O's adaptive computation of state-adaptive improvement-constraint balances and their consequential effects through comprehensive analyses.
Ablation studies on the adaptive balances and balance model further corroborate FamO2O's efficacy. 

The limitation of our work is that FamO2O has been evaluated on just a handful of representative offline-to-online RL algorithms, leaving a vast array of such algorithms unexamined. Additionally, FamO2O’s utility is somewhat limited, as it is applicable exclusively to offline-to-online algorithms.
In future work, we aim to expand the applicability of FamO2O by integrating it with a broader spectrum of offline-to-online algorithms. Additionally, we intend to investigate the feasibility of incorporating state-adaptive improvement-constraint balances in offline RL settings, where online feedback is intrinsically absent, or even applying our adaptive design to the LLM-as-agent domain~\cite{wang2023survey, wang2023avalon}.

%% file: sections/8_appendix.tex
\section{Related Work}
\input{sections/6_related_work}

\section{Training Process of FamO2O}    \label{sec: pseudo-code}
Here we provide the pseudo-codes of FamO2O (see \cref{alg: algos-FamO2O}) to demonstrate its training process.

\begin{algorithm}[!h]
\caption{FamO2O Algorithm} \label{alg: algos-FamO2O}
\begin{algorithmic}[1]
    \Require  replay buffer $\mathcal{D}$, offline dataset $\mathcal{D}_{\text{offline}}$, offline and online training steps $N_{\text{off}}, N_{\text{on}}$
    %
    \State Initialize $\pi_b^0, \pi_u^0, Q^0, V^0$
    \State Initialize $\mathcal{D}$ with $\mathcal{D}_{\text{offline}}$
    \For{$k = 0 \to N_{\text{off}}-1$} \Comment{offline pre-training phase}
        \State Sample a mini-batch $M_k$ from $\mathcal{D}$
        \State Assign a random balance coefficient for each sample in $M_k$, denoting the balance coefficient set as $B_{M_k}$
        \State Update $\pi_u^k$ to $\pi_u^{k+1}$ with $M_k, B_{M_k}$ by \cref{eq: the optimization target of the universal model}
        \State Update $\pi_b^k$ to $\pi_b^{k+1}$ with $M_k$ by \cref{eq: the optimization target of the balance model}
        \State Update $Q^k, V^k$ to $Q^{k+1}, V^{k+1}$ respectively with  $M_k$ according to the base algorithm
    \EndFor
    \For{$k = N_{\text{off}} \to N_{\text{on}}-1$} \Comment{online fine-tuning phase}
        \State Collect samples with $\pi_u^k, \pi_b^k$ and add samples to $\mathcal{D}$ 
        \State Sample a mini-batch $M_k$ from $\mathcal{D}$
        \State $\pi_b^k$ computes a balance coefficient for each sample in $M_k$. Denote the balance coefficient set as $B_{M_k}$
        \State Update $\pi_u^k$ to $\pi_u^{k+1}$ with $M_k, B_{M_k}$ by \cref{eq: the optimization target of the universal model}
        \State Update $\pi_b^k$ to $\pi_b^{k+1}$ with $M_k$ by \cref{eq: the optimization target of the balance model}
        \State Update $Q^k, V^k$ to $Q^{k+1}, V^{k+1}$ respectively with  $M_k$ according to the base algorithm 
    \EndFor
\end{algorithmic}
\end{algorithm}

\section{Theoretical Proofs}
\subsection{Proof of \cref{prop: point-wise superiority}}~\label{appendix: proof of point-wise superiority}
\begin{lemma}\label{lemma: connection between distributional and point-wise constraints}
    \cref{eq: distributional constraint} is equivalent to:
\begin{align}
    &\exists \{\epsilon_\s, \s\in\mathcal{S}\} \in C, \quad \forall \s \in \mathcal{S}, \quad  D_{\mathrm{KL}}(\pi(\cdot | \s) \| \pi_{\beta}(\cdot | \s)) \le \epsilon_\s, \label{eq: equivalent equation 1}\\
    &\text{where} \ C = \left\{\left\{\epsilon_\s^\prime, \s \in \mathcal{S}\right\} | \int_{\s \in\mathcal{S}} d_{\pi_\beta}(\s) \epsilon_\s^\prime \diff\s = \epsilon, \epsilon_\s^\prime\ge 0 \right\}. \label{eq: equivalent equation 2}
\end{align}
\end{lemma}
\begin{proof}
    We aim to prove \cref{eq: distributional constraint} $\Rightarrow$ \crefrange{eq: equivalent equation 1}{eq: equivalent equation 2} in our manuscript through a contradiction approach. Let's assume that \cref{eq: distributional constraint} $\nRightarrow$ \crefrange{eq: equivalent equation 1}{eq: equivalent equation 2}, signifying the following:
\begin{enumerate}
    \item \textbf{Condition 1} \quad According to \cref{eq: distributional constraint}, $\int_{\mathbf{s}\in \mathcal{S}} d_{\pi_\beta(\mathbf{s})} D_{\text{KL}}(\pi(\cdot|\mathbf{s})\| \pi_\beta(\cdot|\mathbf{s})) \mathrm{d}\mathbf{s} \le \epsilon$;
    \item \textbf{Condition 2} \quad The converse proposition of \crefrange{eq: equivalent equation 1}{eq: equivalent equation 2} is: $\forall \epsilon_{\mathbf{s}}\ge 0\ (\mathbf{s}\in \mathcal{S})$ which satisfy $\int_{\mathbf{s}\in\mathcal{S}}d_{\pi_\beta}(\mathbf{s})\epsilon_{\mathbf{s}}\mathrm{d}\mathbf{s}=\epsilon$, there exists $\mathbf{s}\in\mathcal{S}$, $D_{\text{KL}}(\pi(\cdot|\mathbf{s})\| \pi_\beta(\cdot|\mathbf{s})) > \epsilon_{\mathbf{s}}$.
\end{enumerate}

Next, we form a special set $\{\epsilon_{\mathbf{s}}, \mathbf{s} \in \mathcal{S}\}$. Selecting an arbitrary $\mathbf{s}_0 \in \mathcal{S}$, the set fulfills:
\begin{enumerate}
    \item For all $\mathbf{s}\neq \mathbf{s}_0$, $\epsilon_{\mathbf{s}} = D_{\text{KL}}(\pi(\cdot|\mathbf{s})\| \pi_\beta(\cdot|\mathbf{s}))$,
    \item $\epsilon_{\mathbf{s}_0} = \epsilon - \int_{\mathbf{s}\in\mathcal{S}, \mathbf{s}\neq \mathbf{s}_0} d_{\pi_\beta}(\mathbf{s})D_{\text{KL}}(\pi(\cdot|\mathbf{s}_0)\| \pi_\beta(\cdot|\mathbf{s}_0))\mathrm{d}\mathbf{s}$.
\end{enumerate} 

It can be clearly seen that the set adheres to $\int_{\mathbf{s}\in\mathcal{S}}d_{\pi_\beta}(\mathbf{s})\epsilon_{\mathbf{s}}\mathrm{d}\mathbf{s}=\epsilon$. Based on Condition 2, we can deduce that $D_{\text{KL}}(\pi(\cdot|\mathbf{s}_0)\| \pi_\beta(\cdot|\mathbf{s}_0)) > \epsilon_{\mathbf{s}_0}$. Hence,
\begin{align}
\epsilon = &\quad \int_{\mathbf{s}\in\mathcal{S}}d_{\pi_\beta}(\mathbf{s})\epsilon_{\mathbf{s}}\mathrm{d}\mathbf{s} \nonumber\\
=&\quad d_{\pi_\beta}(\mathbf{s}_0)\epsilon_{\mathbf{s}_0} + \int_{\mathbf{s}\in\mathcal{S}, \mathbf{s}\neq\mathbf{s}_0}d_{\pi_\beta}(\mathbf{s})\epsilon_{\mathbf{s}}\mathrm{d}\mathbf{s} \nonumber\\
<&\quad d_{\pi_\beta}(\mathbf{s}_0)D_{\text{KL}}(\pi(\cdot|\mathbf{s}_0)\| \pi_\beta(\cdot|\mathbf{s}_0)) \nonumber
+ \int_{\mathbf{s}\in\mathcal{S}, \mathbf{s}\neq\mathbf{s}_0}d_{\pi_\beta}(\mathbf{s})D_{\text{KL}}(\pi(\cdot|\mathbf{s})\| \pi_\beta(\cdot|\mathbf{s}))\mathrm{d}\mathbf{s} \nonumber\\
=&\quad \int_{\mathbf{s}\in\mathcal{S}}d_{\pi_\beta}(\mathbf{s})D_{\text{KL}}(\pi(\cdot|\mathbf{s})\| \pi_\beta(\cdot|\mathbf{s}))\mathrm{d}\mathbf{s}.\nonumber
\end{align}

This stands in contradiction to condition 1, where $\int_{\mathbf{s}\in \mathcal{S}} d_{\pi_\beta(\mathbf{s})} D_{\text{KL}}(\pi(\cdot|\mathbf{s})\| \pi_\beta(\cdot|\mathbf{s})) \mathrm{d}\mathbf{s} \le \epsilon$. Thus, the assumption that \cref{eq: distributional constraint} $\nRightarrow$ \crefrange{eq: equivalent equation 1}{eq: equivalent equation 2} is proven false, which confirms that \cref{eq: distributional constraint} $\Rightarrow$ \crefrange{eq: equivalent equation 1}{eq: equivalent equation 2}.
\end{proof}

Denote the feasible region of the problem in \cref{def: distributional constraint optimization problem definition} as $\mathcal{F}^k(\epsilon)$, and the feasible region of the problem in \cref{def: point-wise constraint optimization problem definition} as $\mathcal{F}^k(\{\epsilon_\s, \s \in \mathcal{S}\})$.
According to \cref{lemma: connection between distributional and point-wise constraints}, we can infer that:
\begin{equation}
    \mathcal{F}^k(\epsilon) = \bigcup_{\{\epsilon_\s, \s \in \mathcal{S}\} \in C } \mathcal{F}^k(\{\epsilon_\s, \s \in \mathcal{S}\}).
\end{equation}
Considering that the problems in \cref{def: point-wise constraint optimization problem definition} and \cref{def: distributional constraint optimization problem definition} shares the same objective function, we have:
\begin{equation}~\label{eq: feasible region}
    \forall \epsilon \ge 0, \quad \exists \{\epsilon_\s, \s \in \mathcal{S}\} \in C,  \quad  \pi^{k+1}_*[\epsilon] \in \mathcal{F}^k(\{\epsilon_\s, \s \in \mathcal{S}\}),
\end{equation}
where $\pi^{k+1}_*[\epsilon]$ is the optimal solution corresponding to the optimal value $J_*^k[\epsilon]$ for the problem in \cref{def: distributional constraint optimization problem definition}.

Based on \cref{eq: feasible region}, we can derive that
\begin{equation}
    \forall \epsilon \ge 0, \quad \exists \{\epsilon_\s, \s \in \mathcal{S}\} \in C, \quad \exists \pi \in \mathcal{F}^k(\{\epsilon_\s, \s \in \mathcal{S}\}), \quad  J^k_\pi[\{\epsilon_\s, \s \in \mathcal{S}\}] \ge J_*^k[\epsilon].
\end{equation}
Here $J^k_\pi[\{\epsilon_\s, \s \in \mathcal{S}\}]$ is the objective function value in \ref{def: point-wise constraint optimization problem definition} with the solution $\pi$ and KL constraints $\{\epsilon_\s, \s \in \mathcal{S}\}$.

Under the KL constraints $\{\epsilon_\s, \s \in \mathcal{S}\}$, the optimal value $J_*^k[\{\epsilon_\s, \s \in \mathcal{S}\}]$ of the problem in \cref{def: point-wise constraint optimization problem definition} is no less than $J^k_\pi[\{\epsilon_\s, \s \in \mathcal{S}\}]$. Therefore
\begin{equation}
    \exists \{\epsilon_\s, \s \in \mathcal{S}\}, \quad  J_*^k[\{\epsilon_\s, \s \in \mathcal{S}\}] \ge J^k_\pi[\{\epsilon_\s, \s \in \mathcal{S}\}] \ge  J_*^k[\epsilon].
\end{equation}
\cref{prop: point-wise superiority} is proven. \textbf{Q.E.D.}

\subsection{Proof of \cref{prop: State-dependent balance coefficient}}\label{appendix: proof of state-dependent balance coefficient}
Because the state space $\mathcal{S}$ possibly contains an infinite number of states, the optimization problem in \cref{def: point-wise constraint optimization problem definition} is probably a problem with infinite constraints, which is not easy to deal with. 
Therefore, we first start with a simplified version of the optimization problem, where the state space contains only a finite number of states, and further extend the conclusion on the simplified optimization problem to the original one in \cref{def: point-wise constraint optimization problem definition}.

The simplified optimization problem is stated as follows:
\begin{definition}[Simplified optimization problem with point-wise constraints]
\label{def: simplified optimization problem with point-wise constraints}
The simplified optimization problem with point-wise KL constraints on a state space with only  a finite number of  states $\mathcal{S}=\{\s_1, \s_2, \cdots, \s_m\}$ is defined as
\begin{align}
\max_{\pi} &\ \  \mathbb{E}_{\s\sim d_{\pi_\beta}(\cdot), \act \sim \pi(\cdot|\s)} \left[Q^{\pi^k}(\s, \act)-V^{\pi^k}(\s)\right] \label{eq: simplified optimization problem}\\
\text { s.t. } & D_{\mathrm{KL}}(\pi(\cdot | \s_i) \| \pi_{\beta}(\cdot | \s_i)) \leq \epsilon_i, \quad i=1, 2, \cdots, m \label{eq: simplified kl constraint}\\
 & \int_{\act\in\mathcal{A}} \pi(\act | \s_i) \diff\act =1, \quad i=1, 2, \cdots, m \label{eq: simplified normalization constraint}.
\end{align}
\end{definition}
Here, $\pi^k (k\in \mathbb{N})$ denotes the policy at the $k$-th iteration, $\pi_\beta$ is a behavior policy representing the action selection way in the offline dataset or current replay buffer, and $d_{\pi_{\beta}}(\s)$ is the state distribution of $\pi_\beta$. we utilize the optimal solution for \cref{def: simplified optimization problem with point-wise constraints} as $\pi^{k+1}$.

For this simplified optimization problem, we have a lemma below, whose derivation is related to AWR~\cite{awr} and AWAC~\cite{nair2020awac}:
\begin{lemma}
\label{lemma: simplified optimization problem}
Consider the optimization problem in \cref{def: simplified optimization problem with point-wise constraints}.
The optimal solution of $\pi^{k+1}$, denoted as $\pi_*^{k+1}$, satisfies that $\forall \s \in \{\s_1, \s_2, \cdots, \s_m\}, \act \in \mathcal{A}$,
\begin{equation}
    \pi_*^{k+1}(\act|\s) = \frac{\pi_\beta(\act|\s)}{Z_{\s}} \exp\left(\beta_{\s}(Q^{\pi^k}(\s, \act)-V^{\pi^k}(\s))\right).
    \label{eq: simplified adaptive optimal policy}
\end{equation}
\end{lemma}

\begin{proof}
The Lagrangian function $L(\pi, \mathbf{\lambda}, \mathbf{\mu})$ is given by:
\begin{equation}
\begin{aligned}
    L(\pi, \mathbf{\lambda}, \mathbf{\mu})  =& -\mathbb{E}_{\s\sim d_{\pi_\beta}(\cdot), \act \sim \pi(\cdot|\s)} \left[Q^{\pi^k}(\s, \act)-V^{\pi^k}(\s)\right] \\
    & + \sum_{i=1}^m \lambda_i \Big(\int_{\act} \pi(\act|\s_i) \diff\act - 1\Big)  + \sum_{i=1}^m \mu_i \Big( D_{\mathrm{KL}}(\pi(\cdot | \s_i) \| \pi_{\beta}(\cdot | \s_i)) - \epsilon_i \Big),
\end{aligned}
\end{equation}
where $\mathbf{\lambda} = (\lambda_1, \lambda_2, \cdots, \lambda_m), \mathbf{\mu} = (\mu_1, \mu_2, \cdots, \mu_m)$ are Lagrange multipliers.
According to the KKT necessary conditions~\cite{KKT}, the optimal solution satisfies that $\forall \act \in \mathcal{A}, i \in \{ 1, 2, \cdots, m\}$,
\begin{equation}
\begin{aligned}
    \frac{\partial L}{\partial \pi(\act|\s_i)} &= -d_{\pi_{\beta}}(\s_i) \left(Q^{\pi^k}(\s_i, \act)-V^{\pi^k}(\s_i)\right) + \lambda_i  + \mu_i \left(\log\left(\frac{\pi(\act|\s_i)}{\pi_{\beta}(\act|\s_i)}\right) + 1\right) = 0.
\end{aligned}
\end{equation} 
Therefore, the optimal policy $\pi^{k+1}_*$ is:
\begin{equation}
\begin{aligned}
    \pi_*^{k+1}(\act|\s_i) = \frac{\pi_\beta(\act|\s_i)}{Z_{\s_i}} \exp\left(\beta_{\s_i}(Q^{\pi^k}(\s_i, \act)-V^{\pi^k}(\s_i))\right),
\end{aligned}
\end{equation} 
where $Z_{\s_i} = \exp\left(\frac{\lambda_i}{\mu_i} + 1\right), \beta_{\s_i} = \frac{d_{\pi_\beta}(\s_i)}{\mu_i}$ are state-dependent factors. 

Furthermore, due to the constraint in \cref{eq: simplified normalization constraint}, $Z_{\s_i}$ also equals to $\int_{\act\in\mathcal{A}} \pi_\beta(\act|\s_i) \exp\left(\beta_{\s_i}(Q^{\pi^k}(\s_i, \act)-V^{\pi^k}(\s_i))\right) \diff\act$.
\end{proof}


Now we consider extending the conclusion in \cref{lemma: simplified optimization problem}  to the more complex optimization problem in \cref{def: point-wise constraint optimization problem definition}, where the state space $\mathcal{S} = [s_{\min}, s_{\max}]^l$ is a compact set and contains an infinite number of states.
Without loss of generality, suppose that  $\mathcal{S} = [0, 1]$.
The derivation below is easy to transfer to $\mathcal{S} = [s_{\min}, s_{\max}]^l$ with small modifications.

Specifically, we construct a sequence sets $\{B_i, i=0,1,2,\cdots\}$ with each element $B_i = \{ j/2^{i}, j=0, 1, \cdots, 2^{i}\}$.
We can observe that $\{B_i, i=0,1,2,\cdots\}$ satisfies:
\begin{align}
    &\forall i \in \mathbb{N}_+, \quad B_i \subseteq \mathcal{S}; \label{eq: Bi first satisfaction}\\
     &\forall i \in \mathbb{N}_+, \quad B_i \subseteq B_{i+1};\\
     & \forall i \in \mathbb{N}_+, \quad |B_i| = 2^i + 1 < \infty, \quad \text{\textit{i.e.}, all $B_i$ are finite sets, and therefore all $B_i$ are compact};\\
     & \lim_{i\to\infty} \sup_{x\in \mathcal{S}} \inf_{y\in B_i} \| x - y \|_{\infty} = \lim_{i\to\infty} \frac{1}{2^{i+1}} = 0.\label{eq: Bi last satisfaction}
\end{align}

The qualities in \crefrange{eq: Bi first satisfaction}{eq: Bi last satisfaction}, together with the assumption in \cref{prop: State-dependent balance coefficient} that the feasible space constrained by \crefrange{eq: optimization problem}{eq: normalization constraint} is not empty for every $\s \in \mathcal{S}$, meets the prerequisites of the discretization method proposed by \cite{reemtsen1991discretization}.
Set $\pi_*^{k+1}[A]$ as the optimal solution of the following optimization problem:
\begin{align}
\pi_*^{k+1}[A] = &\argmax{\pi}  \mathbb{E}_{\s\sim d_{\pi_\beta}(\cdot), \act \sim \pi(\cdot|\s)} \left[Q^{\pi^k}(\s, \act)-V^{\pi^k}(\s)\right]\\
\text { s.t. } & D_{\mathrm{KL}}(\pi(\cdot | \s) \| \pi_{\beta}(\cdot | \s)) \leq \epsilon_\s, \quad \s \in A\\
 & \int_{\act\in\mathcal{A}} \pi(\act | \s) \diff\act =1, \quad  \s \in A,
\end{align}
where $A$ is a subset of $\mathcal{S}$.
According to the Theorem~2.1 in \cite{reemtsen1991discretization}, 
\begin{equation}
    \pi_*^{k+1}[B_i] \xrightarrow{i \to \infty} \pi_*^{k+1}[\mathcal{S}].
    \label{eq: policy extreme}
\end{equation}
For any $B_i$, because $|B_i|<\infty$, \cref{eq: simplified adaptive optimal policy} holds for $\pi_*^{k+1}[B_i]$.
By combining \cref{eq: simplified adaptive optimal policy} with \cref{eq: policy extreme}, we succeed in proving \cref{eq: adaptive optimal policy} in \cref{prop: State-dependent balance coefficient}.

Furthermore, we utilize a parameterized policy $\pi_{\mathbf{\phi}}$ to approximate the optimal policy $\pi_*^{k+1}$ in \cref{eq: adaptive optimal policy}, \textit{i.e.},
\begin{align}
    \mathbf{\phi} &= \argmin{\mathbf{\phi}} \mathbb{E}_{\s \sim d_{\pi_\beta}(\cdot)}\left[D_{\mathrm{KL}}(\pi_*^{k+1} (\cdot|\s) \| \pi_{\mathbf{\phi}}(\cdot|\s))\right] \\
    &=\argmax{\phi}\mathbb{E}_{\s \sim d_{\pi_\beta}(\cdot)} \left[ \frac{1}{Z_{\s}}  \mathbb{E}_{\act \sim \pi_\beta(\act|\s)} \left[ \exp(\beta_{\s} (Q^{\pi^k}(\s, \act)-V^{\pi^k}(\s))) \log \pi_{\mathbf{\phi}}(\act|\s) \right] \right].\label{eq: approximate to the optimal policy}
\end{align}

In practice, $Z_{\s}$ in \cref{eq: approximate to the optimal policy} is challenging to calculate.
We follow the derivation of AWR~\cite{awr} and AWAC~\cite{nair2020awac}, and omit the term $\frac{1}{Z_{\s}}$.
Therefore \cref{eq: approximate to the optimal policy} can be rewritten as 
\begin{align}
    \mathbf{\phi} &= \argmax{\phi} \mathbb{E}_{\s \sim d_{\pi_\beta}(\cdot)}\left[\mathbb{E}_{\act \sim \pi_\beta(\act|\s)} \left[  \exp(\beta_{\s} (Q^{\pi^k}(\s, \act)-V^{\pi^k}(\s))) \log \pi_{\mathbf{\phi}}(\act|\s) \right] \right] \\
    & = \argmax{\phi}\mathbb{E}_{(\s, \act)\sim \mathcal{D}} \left[ \exp(\beta_{\s} (Q^{\pi^k}(\s, \act)-V^{\pi^k}(\s))) \log \pi_{\mathbf{\phi}}(\act|\s) \right].
\end{align}
Thus \cref{eq: adaptive training objective} in \cref{prop: State-dependent balance coefficient} is also proven.  
\textbf{Q.E.D.}

\section{FamO2O's Extension to Non-AWR Algorithms}\label{sec: FamO2O extension}
In this section, we discuss extending FamO2O to non-AWR algorithms, specifically considering the Conservative Q-Learning (CQL)~\cite{kumar2020cql}.

CQL's policy update rule is given in \cref{eq: cql policy update}:
\begin{equation}
    \pi^{k+1} = \argmax{\pi} \mathbb{E}_{\s\sim \mathcal{D}, \act \sim \pi(\cdot|\s)} \left[ \alpha \cdot {\color{blue}Q^{k}(\s, \act)} - {\color{red}\log\pi(\act | \s)}  \right],
\label{eq: cql policy update}
\end{equation}
where $\pi^k$ and $Q^k$ represent the policy and Q function at iteration $k$ respectively, and $\alpha$ is a hyperparamter.
CQL already incorporates a conservative estimation for out-of-distribution data in its Q function update rule, thus lacking a policy constraint for policy updates in \cref{eq: cql policy update}. However, \cref{eq: cql policy update} presents a balance between exploitation (the {\color{blue}blue} part) and exploration (the {\color{red}red} part), determined by $\alpha$. We thus deploy FamO2O to adjust this balance adaptively per state.

Under FamO2O, the policy update rule morphs to \cref{eq: cql with FamO2O}:
\begin{align}
    &\pi^{k+1}_u = \argmax{\pi_u} \mathbb{E}_{\s\sim \mathcal{D}, \act \sim \pi_u(\cdot|\s, {\color{red}\alpha_\s})} \left[ {\color{blue}\alpha_\s} \cdot  Q^k(\s, \act) -\log\pi_u(\act | \s, {\color{red}\alpha_\s})  \right], \label{eq: cql with FamO2O} \\
    &\text{where}\quad \alpha_\s = \pi_b^k(\s),
\end{align}
where $\pi_u$ is the universal model taking a state and balance coefficient as input to yield an action, and $\pi_b$ is the balance model mapping a state to a balance coefficient. The changes, denoted in {\color{red}red} and {\color{blue}blue}, depict FamO2O's additional input and the substitution of the adaptive balance coefficient $\alpha_\s$ for the pre-defined one $\alpha$, respectively.

Additionally, the balance model update rule (\cref{eq: family cql balance model update}) aims to maximize the corresponding Q value by finding an optimal $\alpha_\s$, akin to \cref{eq: the optimization target of the balance model}.
\begin{equation}
    \pi_b^{k+1} = \argmax{\pi_b} \mathbb{E}_{\s\sim\mathcal{D}} \Big[Q^k(\s, \underbrace{\pi_u^{k+1}(\s, \overbrace{\pi_b(\s)}^{\mathclap{\text{balance coefficient }\beta_\s}})}_{\text{action}})\Big].
    \label{eq: family cql balance model update}
\end{equation}

As for Q function updates, by consolidating balance model $\pi_b: \mathcal{S}\mapsto\mathcal{B}$ and $\pi_u: \mathcal{S}\times \mathcal{B}\mapsto \mathcal{A}$ into a standard policy $\pi_u(\cdot, \pi_b(\cdot)): \mathcal{S}\mapsto \mathcal{A}$, these updates remain identical to standard CQL.

\section{Implementation Details}\label{sec: implementation details}
The following section outlines the specifics of baseline (\cref{appendix: baseline implementation}) and FamO2O (\cref{appendix: FamO2O implementation details}) implementations.

\subsection{Baseline Implementation}\label{appendix: baseline implementation}
We commence by detailing our baseline implementation. For IQL~\cite{iql}, AWAC~\cite{nair2020awac}, and Balanced Replay~\cite{lee2022pessimistic}, which have been tested in an offline-to-online RL setting in their respective papers, we utilize the official codes\footnote{IQL and AWAC: \url{https://github.com/rail-berkeley/rlkit}, commit ID c81509d982b4d52a6239e7bfe7d2540e3d3cd986.}\footnote{Balanced Replay: \url{https://github.com/shlee94/Off2OnRL}, commit ID 6f298fa}.

For CQL~\cite{kumar2020cql} and TD3+BC~\cite{td3+bc}, we introduce an online fine-tuning process based on the recommended offline RL codes by the authors\footnote{CQL: \url{https://github.com/young-geng/JaxCQL}, commit ID 80006e1a3161c0a7162295e7002aebb42cb8c5fa.}\footnote{TD3+BC: \url{https://github.com/sfujim/TD3_BC}, commit ID 8791ad7d7656cb8396f1b3ac8c37f170b2a2dd5f.}. Specifically, our online fine-tuning for CQL and TD3+BC mirrors the implementation in \href{https://github.com/rail-berkeley/rlkit}{rlkit}, that is, appending samples acquired during online interactions to the offline dataset while maintaining the same training objective during offline pre-training and online fine-tuning.

For SAC~\cite{sac} and BC, we leverage CQL's codebase, which has already incorporated these two algorithms.

Regarding training steps, we adhere to IQL's experiment setting for a fair comparison across all methods in \cref{fig: compare-with-baseline}, which involves $10^6$ gradient steps each for offline pre-training and online fine-tuning. It should be noted that due to AWAC's overfitting issue, as discussed by \cite{iql}, we limit AWAC's offline pre-training to $2.5\times 10^4$ gradient steps, as recommended. For non-AWR algorithms comparison, namely CQL~\cite{kumar2020cql}, BR~\cite{lee2022pessimistic}, and CQL+FamO2O, we employ $2\times 10^6$ offline gradient steps, as suggested in CQL's implementation, and $1\times 10^6$ online gradient steps, aligning with IQL.

\subsection{Implementation of FamO2O}\label{appendix: FamO2O implementation details}
Next, we provide FamO2O's implementation details.
We implement our FamO2O on the official codes of AWAC~\cite{nair2020awac}, IQL~\cite{iql}, and CQL~\cite{kumar2020cql} discussed above.
Our major modification is 
\begin{inlinelist}
\item adding a balance model and
\item making the original policy conditioned on the balance coefficient (called ``universal model" in our paper).
\end{inlinelist}
Except for these two modifications, we do not change any training procedures and hyperparameters in the original codes.

For the first modification, we employ the stochastic model used by SAC in \href{https://github.com/rail-berkeley/rlkit}{rlkit} to construct our balance model and implement the training process of the balance model described in \cref{sec: learning balance model}.
For the second modification, we change the input dimension of the original policy from the state dimension $l$ to $l + l_b$, where $l_b$ is the dimension of balance coefficient encodings.
We encode balance coefficients in the same way as the positional encoding proposed by \cite{vaswani2017attention}.
This encoding design expands the dimension of the balance coefficients and avoids the balance coefficients being neglected by the universal model.
%

For IQL with FamO2O, adaptive balance coefficients are no less than $1.0$ (Locomotion) / $8.0$ (AntMaze) and no larger than $5.0$ (Locomotion) / $14.0$ (AntMaze).
For AWAC with FamO2O, the adaptive balance coefficients are no less than $2.0$ (Locomotion) / $9.0$ (AntMaze) and no larger than $3.0$ (Locomotion) / $11.0$ (AntMaze).
For CQL with FamO2O, the adaptive balance coefficients are no less than $0.5$ and no larger than $1.5$.
To ensure a fair comparison, in \cref{fig: fixed-constraint-degree}, the fixed balance coefficient ranges cover the adaptive ranges discussed above.

As for training steps, for IQL with FamO2O, we use $10^6$ gradient steps in the offline pre-training phase and $10^6$ gradient steps in the online fine-tuning phase.
For AWAC with FamO2O, we use $7.5\times 10^4$ gradient steps in the offline pre-training phase and $2\times 10^5$ (Locomotion) / $1\times 10^6$ (AntMaze) gradient steps in the online fine-tuning phase.
For CQL with FamO2O, we use $2\times 10^6$ gradient steps in the offline pre-training phase and $1\times 10^6$ in the online fine-tuning phase.

For more implementation details, please refer to \crefrange{tab: IQL details}{tab: CQL details}.

\section{More Experimental Results}\label{appendix: more results}
\subsection{Sensitivity Analysis on Different Ranges of Balance Coefficients}
To find out whether FamO2O's performance is sensitive to the range of balance coefficients, we implement FamO2O on IQL~\cite{iql} with 3 different balance coefficient ranges (\textit{i.e.}, $[1, 5]$, $[1, 4]$, and $[2, 4]$) on 3 datasets of different qualities (\textit{i.e.}, medium-replay, medium, and medium-expert).
The results can be found in \cref{fig:different ranges}.
We further implement FamO2O on IQL with 4 different balance coefficient ranges (\textit{i.e.}, $[6, 12]$, $[8, 12]$, $[8, 14]$, and $[9, 11]$) on 3 datasets of different maze sizes (\textit{i.e.}, umaze, medium, large). The results can be viewed in \cref{fig: antmaze different ranges}.
It can be observed that on various datasets of different qualities or maze sizes, the performance of FamO2O does not vary significantly with different ranges of the balance coefficients.
This indicates that within a reasonable scope, FamO2O is insensitive to the range of coefficient balances.

\begin{figure}[!h]
    \centering
    \includegraphics[width=1.0\columnwidth]{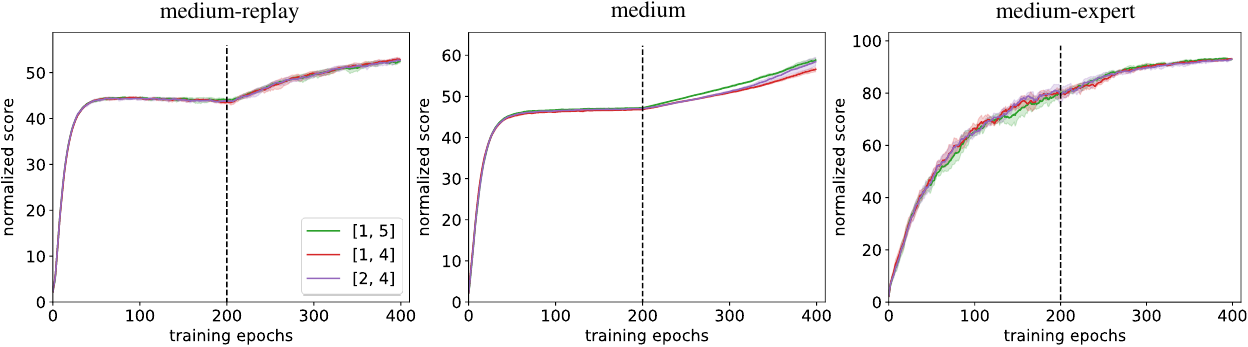}
    \vspace{-10pt}
    \caption{Sensitivity analysis on different ranges of balance coefficients. We choose 3 different ranges (\textit{i.e.}, $[1, 5]$, $[1, 4]$, and $[2, 4]$) and \textbf{3 different dataset qualities} (\textit{i.e.}, medium-replay, medium, and medium-expert) in the D4RL~\cite{fu2020d4rl} HalfCheetah environment. Results are evaluated over 6 random seeds. The shade around curves denotes 95\% CIs of the policy performance. Each epoch contains 5k gradient steps. The dashed lines divide the pre-training and fine-tuning.}
    \label{fig:different ranges}
    \vspace{-5pt}
\end{figure}

\begin{figure}[!h]
    \centering
    \includegraphics[width=1.0\columnwidth]{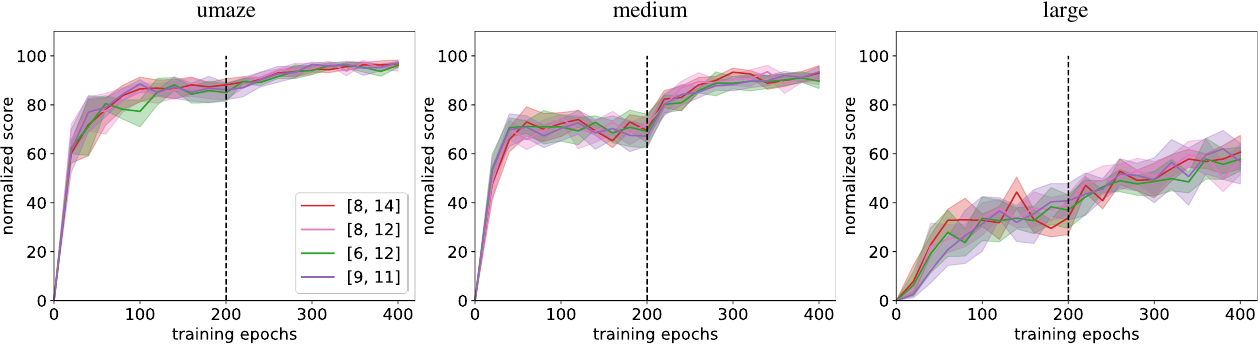}
    \vspace{-10pt}
    \caption{Sensitivity analysis on different ranges of balance coefficients. We choose 4 different ranges (\textit{i.e.}, $[6, 12]$, $[8, 12]$, $[8, 14]$, and $[9, 11]$) and \textbf{3 different maze sizes} (\textit{i.e.}, umaze, medium, and large) in the D4RL~\cite{fu2020d4rl} AntMaze play environments. Results are evaluated over 6 random seeds. The shade around curves denotes 95\% CIs of the policy performance. Each epoch contains 5k gradient steps. The dashed lines divide the pre-training and fine-tuning.}
    \label{fig: antmaze different ranges}
    \vspace{-10pt}
\end{figure}

\subsection{Comparing FamO2O and Balance Coefficient Annealing}

\begin{wraptable}{r}{0.55\textwidth}
\vspace{-23pt}
\begin{threeparttable}
  \small
  \caption{Comparing FamO2O and balance coefficient annealing, FamO2O notably surpasses the latter on the majority of D4RL locomotion datasets~\cite{fu2020d4rl}, demonstrating statistically significant performance improvement.}
  \vspace{10pt}
  \centering
    \renewcommand{\tabcolsep}{6.5pt}
  \renewcommand{\arraystretch}{1.5}
  \begin{tabular}{lcc}
    \specialrule{0.12em}{0pt}{0pt}
	Dataset 
	& FamO2O 
	& Annealing
	\\ \hline
    hopper-mr-v2
    & \colorbox{myblue}{97.64}
    & 87.82
    \\
    hopper-m-v2
    & \colorbox{myblue}{90.65}
    & 72.98
    \\
    hopper-me-v2
    & \colorbox{myblue}{87.28}
    & 72.26
    \\
    halfcheetah-mr-v2
    & 53.07
    & \colorbox{myblue}{53.35}
    \\ 
    halfcheetah-m-v2
    & \colorbox{myblue}{59.15}
    & 57.37
    \\ 
    halfcheetah-me-v2
    & 93.10
    & \colorbox{myblue}{93.91}
    \\ 
    walker2d-mr-v2
    & \colorbox{myblue}{92.85}
    & 87.29
    \\ 
    walker2d-m-v2
    & 85.50
    & \colorbox{myblue}{85.89}
    \\ 
    walker2d-me-v2
    & 112.72
    & 112.72
    \\ 
    \hline
    \textbf{total} 
    & \colorbox{myblue}{\textbf{771.96}}
    & \textbf{723.61} 
    \\
    \textbf{95\% CIs} & \textbf{753.52\textasciitilde788.51} & \textbf{704.21\textasciitilde744.43} \\
    \specialrule{0.12em}{0pt}{0pt}
  \end{tabular}
  \label{tab:anneal-FamO2O}
\end{threeparttable}
\vspace{-15pt}
\end{wraptable}

To delve deeper into the efficacy of FamO2O's balance model more thoroughly, we set up a comparison against the balance coefficient annealing. The balance coefficient annealing is a process that gradually augments the balance coefficient value as the fine-tuning stage proceeds. This method is intuitively reasonable, as the initial conservativeness due to the balance coefficient annealing helps to curb extrapolation errors at the commencement of fine-tuning, and as the process continues, this conservativeness is gradually eased to facilitate policy improvement.

We incorporate balance coefficient annealing into IQL~\cite{iql}, named IQL+annealing, providing a basis for comparison with IQL+FamO2O. The comparison results are presented in \cref{tab:anneal-FamO2O}. The results show that FamO2O exceeds the performance of balance coefficient annealing across most of the D4RL Locomotion datasets~\cite{fu2020d4rl}.

In addition, we calculate the 95\% CIs for both IQL+FamO2O and IQL+annealing in \cref{tab:anneal-FamO2O}, following the suggestion by rliable~\cite{rliable}. The lower CI for IQL+FamO2O is mush higher than the upper CI for IQL+annealing, signifying a statistically significant improvement of FamO2O over balance coefficient annealing. These findings establish that in comparison to annealing, FamO2O showcases a more sophisticated and effective approach in adaptively modulating the balance between improvement and constraint.

\section{Potential Negative Societal Impacts}
This work proposes a plug-in framework, named FamO2O, which can be implemented on top of the existing offline-to-online RL algorithms to further improve policy performance.
Therefore, the potential negative societal impacts of our method are similar to those of the offline-to-online RL.
Generally, FamO2O would greatly improve policy performance.
However, as with most RL algorithms, FamO2O cannot guarantee to take safe and effective actions in all kinds of scenarios.
We advocate that the users of FamO2O should be aware of the potential consequences and utilize FamO2O safely, especially in online environments such as self-driving, robotics, and healthcare.

\begin{table}[!h]
\caption{Details of FamO2O implemented on \textbf{IQL}~\cite{iql} for the \textbf{D4RL Locomotion and AntMaze} benchmark~\cite{fu2020d4rl}. Except for our newly added balance model, the hyperparameters of our method keep the same as those of the vanilla IQL.}
\label{tab: IQL details}
\centering
\vspace{5pt}
\begin{tabular}{cll}
    \toprule
    & Name & Value \\
    \midrule
    \multirow{5}{*}{Balance model} & optimizer & Adam~\cite{Adam} \\
    & learning rate & $3\times 10^{-4}$ \\
    & update frequency & 5 (Locomotion) / 1 (AntMaze) \\
    & hidden units & $[256, 256]$\\
    & activation & ReLU~\cite{relu}\\
    \midrule
    \multirow{5}{*}{Universal model} & optimizer & Adam~\cite{Adam}\\
    & learning rate & $3\times 10^{-4}$ \\
    & update frequency & 1 \\
    & hidden units & $[256, 256]$ \\
    & activation & ReLU~\cite{relu} \\
    \midrule
    \multirow{6}{*}{$Q$ function model} & optimizer & Adam~\cite{Adam} \\
    & learning rate & $3\times 10^{-4}$ \\
    & update frequency & 1 \\
    & hidden units & $[256, 256]$ \\
    & activation & ReLU~\cite{relu} \\
    & target $Q$ soft update rate & $5\times 10^{-3}$\\
    \midrule
    \multirow{7}{*}{$V$ function model} & optimizer & Adam~\cite{Adam} \\
    & learning rate & $3\times 10^{-4}$ \\
    & update frequency & 1 \\
    & hidden units & $[256, 256]$ \\
    & activation & ReLU~\cite{relu} \\
    & target $V$ soft update rate & $5\times 10^{-3}$\\
    & quantile & 0.7 (Locomotion) / 0.9 (AntMaze) \\
    \midrule
    \multirow{3}{*}{Other training parameters} & batch size & 256 \\
    & replay buffer size & $2 \times 10^6$ \\
    & discount & 0.99 \\
    \bottomrule
  \end{tabular}
\end{table}

\begin{table}[t]
\caption{Details of FamO2O implemented on \textbf{AWAC}~\cite{nair2020awac} for the \textbf{D4RL Locomotion and AntMaze} benchmark~\cite{fu2020d4rl}. Except for our newly added balance model, the hyperparameters of our method keep the same as those of the vanilla AWAC.}
\label{tab: AWAC details}
\centering
\vspace{5pt}
\begin{tabular}{cll}
    \toprule
    & Name & Value \\
    \midrule
    \multirow{5}{*}{Balance model} & optimizer & Adam~\cite{Adam} \\
    & learning rate & $3\times 10^{-4}$ \\
    & update frequency & 5 \\
    & hidden units & $[256, 256]$\\
    & activation & ReLU~\cite{relu}\\
    \midrule
    \multirow{6}{*}{Universal model} & optimizer & Adam~\cite{Adam}\\
    & learning rate & $3\times 10^{-4}$ \\
    & update frequency & 1 \\
    & hidden units & $[256, 256, 256, 256]$ \\
    & activation & ReLU~\cite{relu} \\
    & weight decay & $1\times 10^4$\\
    \midrule
    \multirow{6}{*}{$Q$ function model} & optimizer & Adam~\cite{Adam} \\
    & learning rate & $3\times 10^{-4}$ \\
    & update frequency & 1 \\
    & hidden units & $[256, 256]$ \\
    & activation & ReLU~\cite{relu} \\
    & target $Q$ soft update rate & $5\times 10^{-3}$\\
    \midrule
    \multirow{3}{*}{Other training parameters} & batch size & 1024 \\
    & replay buffer size & $2 \times 10^6$ \\
    & discount & 0.99 \\
    \bottomrule
  \end{tabular}
\end{table}

\begin{table}[t]
\caption{Details of FamO2O implemented on \textbf{CQL}~\cite{kumar2020cql} for the \textbf{D4RL Locomotion} benchmark~\cite{fu2020d4rl}. Except for our newly added balance model, the hyperparameters of our method keep the same as those of the vanilla AWAC.}
\label{tab: CQL details}
\centering
\vspace{5pt}
\begin{tabular}{cll}
    \toprule
    & Name & Value \\
    \midrule
    \multirow{5}{*}{Balance model} & optimizer & Adam~\cite{Adam} \\
    & learning rate & $3\times 10^{-4}$ \\
    & update frequency & 1 \\
    & hidden units & $[256, 256]$\\
    & activation & ReLU~\cite{relu}\\
    \midrule
    \multirow{6}{*}{Universal model} & optimizer & Adam~\cite{Adam}\\
    & learning rate & $3\times 10^{-4}$ \\
    & update frequency & 1 \\
    & hidden units & $[256, 256]$ \\
    & activation & ReLU~\cite{relu} \\
    \midrule
    \multirow{6}{*}{$Q$ function model} & optimizer & Adam~\cite{Adam} \\
    & learning rate & $3\times 10^{-4}$ \\
    & update frequency & 1 \\
    & hidden units & $[256, 256]$ \\
    & activation & ReLU~\cite{relu} \\
    & target $Q$ soft update rate & $5\times 10^{-3}$\\
    \midrule
    \multirow{3}{*}{Other training parameters} & batch size & 256 \\
    & replay buffer size & $2 \times 10^6$ \\
    & discount & 0.99 \\
    \bottomrule
  \end{tabular}
\end{table}

%% file: sections/6_related_work.tex
\paragraph{Offline RL} 
Reinforcement learning (RL) usually improves a policy through continuous online interactions with the environment ~\cite{sutton1998, dqn, yue2023vcr}. 
To reduce the huge demand for online interactions, especially when they are costly or risky, researchers proposed offline RL that utilizes pre-collected datasets to improve the policy without interacting with the environment~\cite{Sergey2020tutorial}.
Directly applying off-policy algorithms in offline RL usually leads to poor performance. This phenomenon is due to \textit{distributional shift}, where the agent may learn inaccurate value estimates of out-of-distribution (OOD) state-action pairs~\cite{Sergey2020tutorial, yue2022boosting, yue2023offline}.
%
%
Existing algorithms address this issue via policy constraints~\cite{fujimoto2019BCQ, kumar2019stabilizing, wu2019behavior, Kostrikov2021fisher-BRC, awr, nair2020awac}, importance sampling~\cite{jiang2016doubly, hallak2017consistent, gendice2020}, regularization~\cite{algaedice2019, kumar2020cql}, uncertainty estimation~\cite{agarwal2020optimistic}, and imitation learning~\cite{chen2020bail, keepdoing2020, wang2020critic}.

\paragraph{Offline-to-online RL} 
Offline-to-online RL adds online fine-tuning to offline RL to enhance policy performance. 
Similar to the distributional shift problem in offline RL, offline-to-online (O2O) RL also suffers from off-policy bootstrapping error accumulation caused by OOD data, which causes a large ``dip" in the initial performance of online fine-tuning ~\cite{nair2020awac}. Moreover, O2O RL algorithms tend to be excessively conservative and result in plateauing performance improvement \cite{nair2020awac}. This emphasizes the importance of finding a balance between being optimistic about improving the policy during the online phase and still being constrained to the conservative offline policy. However, as previously discussed, existing O2O RL algorithms generally adopt a “one-size-fits-all” trade-off, either in the process of Q-learning~\cite{kumar2020cql,lyu2022mildly,iql,mark2022finetune}, along with the policy improvement~\cite{nair2020awac}, via a revised replay buffer~\cite{lee2022pessimistic}, or through some alignment approaches~\cite{wu2022supported,anonymous2023actorcritic}.
It's worth noting that, \cite{td3+bc+finetune} also proposes an adaptive weight in online fine-tuning.
However, the adaptive weight is a single parameter for all the samples. In addition, its update requires an extra human-defined target score, which is absent in the setting considered by our competitors and us.
Furthermore, \cite{PEX} presents a policy expansion strategy that adaptively selects a policy from a set, which includes both offline pre-trained policies and online learned policies.
Nevertheless, this policy set comprises only two policies, in contrast to the infinite number of policies contained in FamO2O's policy family.

\paragraph{Conditioned RL policy} Algorithms in this scheme can be divided into two stages:
\begin{inlinelist}
\item a set of policies is trained;
\item the policy with the highest performance is selected.
\end{inlinelist}
Generally, these methods learn a family of policies that exhibit different behaviors and search for the best one that successfully executes the desired task~\cite{RoboImitationPeng20,Yu2019}.
A similar method to our FamO2O is UP-OSI ~\cite{UP-OSI2017}, where a Universal Policy (UP) within a parameterized space of dynamic models and an Online System Identification (OSI) function that predicts dynamic model parameters based on recent state and action history are proposed. 
Unlike these prior methods utilizing models to identify the environmental parameters, the balance model in our method aims to determine a proper balance between policy improvement and constraints for each state.
It's worth noting that \cite{hong2023confidenceconditioned, swazinna2023userinteractive, ghosh2022offline} also exhibit conditioned designs. 
However, \cite{hong2023confidenceconditioned} focuses on conditioned value functions and pure offline RL; \cite{swazinna2023userinteractive} relies on human interactions; \cite{ghosh2022offline} concentrates on adjusting offline RL policies according to the estimated probabilities of various MDPs in uncovered areas, but uncoverage might not be a primary concern in offline-to-online RL due to the accessibility to the online environment.
Moreover, \cite{yang2023hundreds, yang2023boosting} both exhibit a comparable understanding of how to adjust policy improvement and policy constraint across samples of varying quality.
Nevertheless, the focus of \cite{yang2023hundreds, yang2023boosting} is on offline RL. Specifically, \cite{yang2023hundreds} requires a small number of expert demonstrations, while \cite{yang2023boosting} necessitates annotated preferences to appropriately calibrate the balance for each individual sample.
Additionally, in areas such as computer vision, there are existing works that condition their training and/or inference on various input samples~\cite{han2022learning, han2022latency, pu2023adaptive, pu2023fine, pu2023rank, Wang_2021_CVPR}. These studies align with the adaptive approach in our work.

\paragraph{Hierarchical RL} Hierarchical Reinforcement Learning (HRL) decomposes complex tasks into manageable subtasks, enhancing sample efficiency, interpretability, and generalization in RL. HRL approaches bifurcate into predefined and learned options (or skills). Predefined options, guided by domain knowledge or human demonstrations, shape high-level policy or constrain action space, exemplified by MAXQ~\cite{maxQ}, options framework~\cite{options-framework}, skill chaining~\cite{skill-chain}, and feudal RL~\cite{dayan1992feudal}. Learned options, optimizing objectives like entropy or reward, are seen in option-critic, skill-discovering~\cite{skill-chain}, feudal networks~\cite{vezhnevets2017feudal}, HEXQ~\cite{HEXQ}, DIAYN~\cite{eysenbach2018diversity}, HIRO~\cite{nachum2018data}, and HIDIO~\cite{zhang2021hierarchical}. While FamO2O shares some design elements with HRL, such as hierarchical policy decision-making, the objectives diverge: FamO2O focuses on balancing policy improvement and constraints per state, while HRL seeks to simplify complex tasks via goals or options.

%% file: main.bbl
\begin{thebibliography}{10}

\bibitem{relu}
Abien~Fred Agarap.
\newblock Deep learning using rectified linear units (relu).
\newblock {\em arXiv preprint arXiv:1803.08375}, 2018.

\bibitem{agarwal2020optimistic}
Rishabh Agarwal, Dale Schuurmans, and Mohammad Norouzi.
\newblock An optimistic perspective on offline reinforcement learning.
\newblock In {\em International Conference on Machine Learning}, 2020.

\bibitem{rliable}
Rishabh Agarwal, Max Schwarzer, Pablo~Samuel Castro, Aaron Courville, and
  Marc~G Bellemare.
\newblock Deep reinforcement learning at the edge of the statistical precipice.
\newblock In {\em Advances in Neural Information Processing Systems}, 2021.

\bibitem{chen2020bail}
Xinyue Chen, Zijian Zhou, Zheng Wang, Che Wang, Yanqiu Wu, and Keith Ross.
\newblock Bail: Best-action imitation learning for batch deep reinforcement
  learning.
\newblock In {\em Advances in Neural Information Processing Systems}, 2020.

\bibitem{dayan1992feudal}
Peter Dayan and Geoffrey~E Hinton.
\newblock Feudal reinforcement learning.
\newblock {\em Advances in neural information processing systems}, 1992.

\bibitem{maxQ}
Thomas~G. Dietterich.
\newblock Hierarchical reinforcement learning with the maxq value function
  decomposition.
\newblock {\em Journal of Artificial Intelligence Research}, 2000.

\bibitem{eysenbach2018diversity}
Benjamin Eysenbach, Abhishek Gupta, Julian Ibarz, and Sergey Levine.
\newblock Diversity is all you need: Learning skills without a reward function.
\newblock {\em arXiv preprint arXiv:1802.06070}, 2018.

\bibitem{fu2020d4rl}
Justin Fu, Aviral Kumar, Ofir Nachum, George Tucker, and Sergey Levine.
\newblock D4rl: Datasets for deep data-driven reinforcement learning, 2020.

\bibitem{fu2019diagnosing}
Justin Fu, Aviral Kumar, Matthew Soh, and Sergey Levine.
\newblock Diagnosing bottlenecks in deep q-learning algorithms.
\newblock In {\em International Conference on Machine Learning}, 2019.

\bibitem{td3+bc}
Scott Fujimoto and Shixiang~Shane Gu.
\newblock A minimalist approach to offline reinforcement learning.
\newblock In {\em Advances in Neural Information Processing Systems}, 2021.

\bibitem{fujimoto2019BCQ}
Scott Fujimoto, David Meger, and Doina Precup.
\newblock Off-policy deep reinforcement learning without exploration.
\newblock In {\em International conference on machine learning}, 2019.

\bibitem{ghosh2022offline}
Dibya Ghosh, Anurag Ajay, Pulkit Agrawal, and Sergey Levine.
\newblock Offline rl policies should be trained to be adaptive.
\newblock In {\em International Conference on Machine Learning}, 2022.

\bibitem{sac}
Tuomas Haarnoja, Aurick Zhou, Pieter Abbeel, and Sergey Levine.
\newblock Soft actor-critic: Off-policy maximum entropy deep reinforcement
  learning with a stochastic actor.
\newblock In {\em International conference on machine learning}, 2018.

\bibitem{hallak2017consistent}
Assaf Hallak and Shie Mannor.
\newblock Consistent on-line off-policy evaluation.
\newblock In {\em International Conference on Machine Learning}, 2017.

\bibitem{han2022learning}
Yizeng Han, Yifan Pu, Zihang Lai, Chaofei Wang, Shiji Song, Junfeng Cao, Wenhui
  Huang, Chao Deng, and Gao Huang.
\newblock Learning to weight samples for dynamic early-exiting networks.
\newblock In {\em European Conference on Computer Vision}, 2022.

\bibitem{han2022latency}
Yizeng Han, Zhihang Yuan, Yifan Pu, Chenhao Xue, Shiji Song, Guangyu Sun, and
  Gao Huang.
\newblock Latency-aware spatial-wise dynamic networks.
\newblock In {\em Advances in Neural Information Processing Systems}, 2022.

\bibitem{HEXQ}
Bernhard Hengst.
\newblock Discovering hierarchy in reinforcement learning with hexq.
\newblock In {\em International Conference on Machine Learning}, 2002.

\bibitem{hong2023confidenceconditioned}
Joey Hong, Aviral Kumar, and Sergey Levine.
\newblock Confidence-conditioned value functions for offline reinforcement
  learning.
\newblock In {\em International Conference on Learning Representations}, 2023.

\bibitem{jiang2016doubly}
Nan Jiang and Lihong Li.
\newblock Doubly robust off-policy value evaluation for reinforcement learning.
\newblock In {\em International Conference on Machine Learning}, 2016.

\bibitem{KKT}
William Karush.
\newblock {\em Minima of Functions of Several Variables with Inequalities as
  Side Conditions}.
\newblock Springer Basel, 2014.

\bibitem{Adam}
Diederik~P. Kingma and Jimmy Ba.
\newblock Adam: {A} method for stochastic optimization.
\newblock In {\em International Conference on Learning Representations}, 2015.

\bibitem{skill-chain}
George Konidaris and Andrew Barto.
\newblock Skill discovery in continuous reinforcement learning domains using
  skill chaining.
\newblock {\em Advances in neural information processing systems}, 2009.

\bibitem{Kostrikov2021fisher-BRC}
Ilya Kostrikov, Rob Fergus, Jonathan Tompson, and Ofir Nachum.
\newblock Offline reinforcement learning with fisher divergence critic
  regularization.
\newblock In {\em International Conference on Machine Learning}, 2021.

\bibitem{iql}
Ilya Kostrikov, Ashvin Nair, and Sergey Levine.
\newblock Offline reinforcement learning with implicit q-learning.
\newblock In {\em International Conference on Learning Representations}, 2022.

\bibitem{kumar2019stabilizing}
Aviral Kumar, Justin Fu, Matthew Soh, George Tucker, and Sergey Levine.
\newblock Stabilizing off-policy q-learning via bootstrapping error reduction.
\newblock In {\em Advances in Neural Information Processing Systems}, 2019.

\bibitem{kumar2020discor}
Aviral Kumar, Abhishek Gupta, and Sergey Levine.
\newblock Discor: Corrective feedback in reinforcement learning via
  distribution correction.
\newblock In {\em Advances in Neural Information Processing Systems}, 2020.

\bibitem{kumar2020cql}
Aviral Kumar, Aurick Zhou, George Tucker, and Sergey Levine.
\newblock Conservative q-learning for offline reinforcement learning.
\newblock In {\em Advances in Neural Information Processing Systems}, 2020.

\bibitem{lee2022pessimistic}
Seunghyun Lee, Younggyo Seo, Kimin Lee, Pieter Abbeel, and Jinwoo Shin.
\newblock Offline-to-online reinforcement learning via balanced replay and
  pessimistic q-ensemble.
\newblock In {\em Conference on Robot Learning}, 2022.

\bibitem{Sergey2020tutorial}
Sergey Levine, Aviral Kumar, George Tucker, and Justin Fu.
\newblock Offline reinforcement learning: Tutorial, review, and perspectives on
  open problems.
\newblock {\em arXiv preprint arXiv:2005.01643}, 2020.

\bibitem{lyu2022mildly}
Jiafei Lyu, Xiaoteng Ma, Xiu Li, and Zongqing Lu.
\newblock Mildly conservative q-learning for offline reinforcement learning.
\newblock In {\em Advances in Neural Information Processing Systems}, 2022.

\bibitem{mark2022finetune}
Max~Sobol Mark, Ali Ghadirzadeh, Xi~Chen, and Chelsea Finn.
\newblock Fine-tuning offline policies with optimistic action selection.
\newblock In {\em Deep Reinforcement Learning Workshop NeurIPS}, 2022.

\bibitem{dqn}
Volodymyr Mnih, Koray Kavukcuoglu, David Silver, Andrei~A Rusu, Joel Veness,
  Marc~G Bellemare, Alex Graves, Martin Riedmiller, Andreas~K Fidjeland, Georg
  Ostrovski, et~al.
\newblock Human-level control through deep reinforcement learning.
\newblock {\em nature}, 518(7540):529--533, 2015.

\bibitem{algaedice2019}
Ofir Nachum, Bo~Dai, Ilya Kostrikov, Yinlam Chow, Lihong Li, and Dale
  Schuurmans.
\newblock Algaedice: Policy gradient from arbitrary experience.
\newblock {\em arXiv preprint arXiv:1912.02074}, 2019.

\bibitem{nachum2018data}
Ofir Nachum, Shixiang~Shane Gu, Honglak Lee, and Sergey Levine.
\newblock Data-efficient hierarchical reinforcement learning.
\newblock {\em Advances in neural information processing systems}, 2018.

\bibitem{nair2020awac}
Ashvin Nair, Abhishek Gupta, Murtaza Dalal, and Sergey Levine.
\newblock Awac: Accelerating online reinforcement learning with offline
  datasets.
\newblock {\em arXiv preprint arXiv:2006.09359}, 2020.

\bibitem{RoboImitationPeng20}
Xue~Bin Peng, Erwin Coumans, Tingnan Zhang, Tsang-Wei~Edward Lee, Jie Tan, and
  Sergey Levine.
\newblock Learning agile robotic locomotion skills by imitating animals.
\newblock In {\em Robotics: Science and Systems}, 2020.

\bibitem{awr}
Xue~Bin Peng, Aviral Kumar, Grace Zhang, and Sergey Levine.
\newblock Advantage-weighted regression: Simple and scalable off-policy
  reinforcement learning.
\newblock {\em arXiv preprint arXiv:1910.00177}, 2019.

\bibitem{pu2023fine}
Yifan Pu, Yizeng Han, Yulin Wang, Junlan Feng, Chao Deng, and Gao Huang.
\newblock Fine-grained recognition with learnable semantic data augmentation.
\newblock {\em arXiv preprint arXiv:2309.00399}, 2023.

\bibitem{pu2023rank}
Yifan Pu, Weicong Liang, Yiduo Hao, Yuhui Yuan, Yukang Yang, Chao Zhang, Han
  Hu, and Gao Huang.
\newblock Rank-detr for high quality object detection.
\newblock In {\em Advances in Neural Information Processing Systems}, 2023.

\bibitem{pu2023adaptive}
Yifan Pu, Yiru Wang, Zhuofan Xia, Yizeng Han, Yulin Wang, Weihao Gan, Zidong
  Wang, Shiji Song, and Gao Huang.
\newblock Adaptive rotated convolution for rotated object detection.
\newblock In {\em International Conference on Computer Vision}, 2023.

\bibitem{reemtsen1991discretization}
Rembert Reemtsen.
\newblock Discretization methods for the solution of semi-infinite programming
  problems.
\newblock {\em Journal of Optimization Theory and Applications}, 1991.

\bibitem{keepdoing2020}
Noah~Y. Siegel, Jost~Tobias Springenberg, Felix Berkenkamp, Abbas Abdolmaleki,
  Michael Neunert, Thomas Lampe, Roland Hafner, Nicolas Heess, and Martin~A.
  Riedmiller.
\newblock Keep doing what worked: Behavior modelling priors for offline
  reinforcement learning.
\newblock In {\em International Conference on Learning Representations}, 2020.

\bibitem{sutton1998}
Richard~S Sutton and Andrew~G Barto.
\newblock {\em Introduction to reinforcement learning}.
\newblock MIT press Cambridge, 1998.

\bibitem{options-framework}
Richard~S. Sutton, Doina Precup, and Satinder Singh.
\newblock Between mdps and semi-mdps: A framework for temporal abstraction in
  reinforcement learning.
\newblock {\em Artificial Intelligence}, 1999.

\bibitem{swazinna2023userinteractive}
Phillip Swazinna, Steffen Udluft, and Thomas Runkler.
\newblock User-interactive offline reinforcement learning.
\newblock In {\em The Eleventh International Conference on Learning
  Representations}, 2023.

\bibitem{tSNE}
Laurens van~der Maaten and Geoffrey Hinton.
\newblock Visualizing data using t-sne.
\newblock {\em Journal of Machine Learning Research}, 2008.

\bibitem{vaswani2017attention}
Ashish Vaswani, Noam Shazeer, Niki Parmar, Jakob Uszkoreit, Llion Jones,
  Aidan~N Gomez, {\L}ukasz Kaiser, and Illia Polosukhin.
\newblock Attention is all you need.
\newblock In {\em Advances in neural information processing systems}, 2017.

\bibitem{vezhnevets2017feudal}
Alexander~Sasha Vezhnevets, Simon Osindero, Tom Schaul, Nicolas Heess, Max
  Jaderberg, David Silver, and Koray Kavukcuoglu.
\newblock Feudal networks for hierarchical reinforcement learning.
\newblock In {\em International Conference on Machine Learning}, 2017.

\bibitem{wang2023survey}
Lei Wang, Chen Ma, Xueyang Feng, Zeyu Zhang, Hao Yang, Jingsen Zhang, Zhiyuan
  Chen, Jiakai Tang, Xu~Chen, Yankai Lin, et~al.
\newblock A survey on large language model based autonomous agents.
\newblock {\em arXiv preprint arXiv:2308.11432}, 2023.

\bibitem{wang2023avalon}
Shenzhi Wang, Chang Liu, Zilong Zheng, Siyuan Qi, Shuo Chen, Qisen Yang, Andrew
  Zhao, Chaofei Wang, Shiji Song, and Gao Huang.
\newblock Avalon's game of thoughts: Battle against deception through recursive
  contemplation.
\newblock {\em arXiv preprint arXiv:2310.01320}, 2023.

\bibitem{Wang_2021_CVPR}
Shenzhi Wang, Liwei Wu, Lei Cui, and Yujun Shen.
\newblock Glancing at the patch: Anomaly localization with global and local
  feature comparison.
\newblock In {\em Proceedings of the IEEE/CVF Conference on Computer Vision and
  Pattern Recognition}, 2021.

\bibitem{wang2020critic}
Ziyu Wang, Alexander Novikov, Konrad Zolna, Josh~S Merel, Jost~Tobias
  Springenberg, Scott~E Reed, Bobak Shahriari, Noah Siegel, Caglar Gulcehre,
  Nicolas Heess, et~al.
\newblock Critic regularized regression.
\newblock In {\em Advances in Neural Information Processing Systems}, 2020.

\bibitem{wu2022supported}
Jialong Wu, Haixu Wu, Zihan Qiu, Jianmin Wang, and Mingsheng Long.
\newblock Supported policy optimization for offline reinforcement learning.
\newblock In {\em Advances in Neural Information Processing Systems}, 2022.

\bibitem{wu2019behavior}
Yifan Wu, George Tucker, and Ofir Nachum.
\newblock Behavior regularized offline reinforcement learning.
\newblock {\em arXiv preprint arXiv:1911.11361}, 2019.

\bibitem{yang2023boosting}
Qisen Yang, Shenzhi Wang, Matthieu~Gaetan Lin, Shiji Song, and Gao Huang.
\newblock Boosting offline reinforcement learning with action preference query.
\newblock {\em arXiv preprint arXiv:2306.03362}, 2023.

\bibitem{yang2023hundreds}
Qisen Yang, Shenzhi Wang, Qihang Zhang, Gao Huang, and Shiji Song.
\newblock Hundreds guide millions: Adaptive offline reinforcement learning with
  expert guidance.
\newblock {\em arXiv preprint arXiv:2309.01448}, 2023.

\bibitem{Yu2019}
Wenhao Yu, C.~Karen Liu, and Greg Turk.
\newblock Policy transfer with strategy optimization.
\newblock In {\em International Conference on Learning Representations}, 2019.

\bibitem{UP-OSI2017}
Wenhao Yu, Jie Tan, C.~Karen Liu, and Greg Turk.
\newblock Preparing for the unknown: Learning a universal policy with online
  system identification.
\newblock In {\em Robotics: Science and Systems}, 2017.

\bibitem{anonymous2023actorcritic}
Zishun Yu and Xinhua Zhang.
\newblock Actor-critic alignment for offline-to-online reinforcement learning.
\newblock In {\em International Conference on Machine Learning}, 2023.

\bibitem{yue2023offline}
Yang Yue, Bingyi Kang, Xiao Ma, Gao Huang, Shiji Song, and Shuicheng Yan.
\newblock Offline prioritized experience replay.
\newblock {\em arXiv preprint arXiv:2306.05412}, 2023.

\bibitem{yue2022boosting}
Yang Yue, Bingyi Kang, Xiao Ma, Zhongwen Xu, Gao Huang, and Shuicheng Yan.
\newblock Boosting offline reinforcement learning via data rebalancing.
\newblock {\em arXiv preprint arXiv:2210.09241}, 2022.

\bibitem{yue2023vcr}
Yang Yue, Bingyi Kang, Zhongwen Xu, Gao Huang, and Shuicheng Yan.
\newblock Value-consistent representation learning for data-efficient
  reinforcement learning.
\newblock In {\em AAAI}, 2023.

\bibitem{PEX}
Haichao Zhang, Wei Xu, and Haonan Yu.
\newblock Policy expansion for bridging offline-to-online reinforcement
  learning.
\newblock In {\em International Conference on Learning Representations
  ({ICLR})}, 2023.

\bibitem{zhang2021hierarchical}
Jesse Zhang, Haonan Yu, and Wei Xu.
\newblock Hierarchical reinforcement learning by discovering intrinsic options.
\newblock In {\em International Conference on Learning Representations}, 2021.

\bibitem{gendice2020}
Ruiyi Zhang, Bo~Dai, Lihong Li, and Dale Schuurmans.
\newblock Gendice: Generalized offline estimation of stationary values.
\newblock {\em arXiv preprint arXiv:2002.09072}, 2020.

\bibitem{td3+bc+finetune}
Yi~Zhao, Rinu Boney, Alexander Ilin, Juho Kannala, and Joni Pajarinen.
\newblock Adaptive behavior cloning regularization for stable offline-to-online
  reinforcement learning.
\newblock {\em arXiv preprint arXiv:2210.13846}, 2022.

\end{thebibliography}
